\documentclass{article}

    \PassOptionsToPackage{numbers, compress}{natbib}

\usepackage[preprint]{neurips_2026}

\makeatletter
\renewcommand{\@noticestring}{}
\makeatother

\usepackage[utf8]{inputenc} 
\usepackage[T1]{fontenc}    
\usepackage{hyperref}       
\usepackage{url}            
\usepackage{booktabs}       
\usepackage{amsfonts}       
\usepackage{amsthm}
\usepackage{nicefrac}       
\usepackage{microtype}      
\usepackage{xcolor}         
\usepackage{multicol}
\usepackage[table]{xcolor}

\usepackage{graphicx}
\usepackage{soul}
\usepackage{comment}
\usepackage{bm}
\usepackage{amsmath}
\usepackage{subcaption}
\usepackage{wrapfig}
\usepackage{caption}
\usepackage{algorithm}
\usepackage{listings}
\usepackage{xcolor}

\usepackage{multirow}

\definecolor{mfcomment}{RGB}{65,125,130}
\definecolor{mffunc}{RGB}{215,35,115}
\definecolor{mfop}{RGB}{25,60,180}

\lstdefinelanguage{messipseudo}{
  morekeywords=[1]{
    enc,router,experts,mix,clf,cross_entropy,
    mean_router_mass,pairwise_responsibility,
    subset_align,sparse,balance,route_consistency,diversity
  },
  keywordstyle=[1]\color{mffunc},
  morecomment=[l]{\#},
  commentstyle=\color{mfcomment},
  literate=
    {*}{{{\color{mfop}*}}}1
    {+}{{{\color{mfop}+}}}1
    {-}{{{\color{mfop}-}}}1
    {=}{{{\color{mfop}=}}}1
}

\lstdefinestyle{meanflowlike}{
  language=messipseudo,
  basicstyle=\ttfamily\small,
  columns=fullflexible,
  keepspaces=true,
  showstringspaces=false,
  numbers=none,
  frame=none,
  xleftmargin=0.2em,
  aboveskip=2pt,
  belowskip=2pt,
  framexleftmargin=0pt,
  framexrightmargin=0pt,
  framesep=3pt
}

\usepackage{amsmath,amsfonts,bm}

\def\eqref#1{equation~\ref{#1}}

\def\1{\bm{1}}

\DeclareMathAlphabet{\mathsfit}{\encodingdefault}{\sfdefault}{m}{sl}
\SetMathAlphabet{\mathsfit}{bold}{\encodingdefault}{\sfdefault}{bx}{n}

\newcommand{\ms}[1]{\ensuremath{\pm #1}}

\newtheorem{theorem}{Theorem}[section]
\newtheorem{proposition}[theorem]{Proposition}
\newtheorem{lemma}[theorem]{Lemma}
\newtheorem{corollary}[theorem]{Corollary}
\newtheorem{definition}[theorem]{Definition}

\title{Learning Subset-Shared Invariances for \\ Domain Generalization with Mixture-of-Experts}

\author{%
Tien-Hung Nguyen\thanks{Co-first Authors.},
Tien-Dat Tran\footnotemark[1],
M.-Duong Nguyen,
Kok-Seng Wong\thanks{Corresponding Author: \url{wong.ks@vinuni.edu.vn}}
\\
VinUniversity, Vietnam
\\
\texttt{\{26hung.nt, duong.nm2, dat.tt6, wong.ks\}@vinuni.edu.vn}
}

\begin{document}

\maketitle

\begin{abstract}
Domain generalization (DG) aims to learn a model from one or more source domains that generalizes to an unseen target domain without accessing target data during training. A common approach enforces invariance of representations across all source domains, assuming predictive structure is globally shared. However, we demonstrate that enforcing invariance across more domains gradually restricts the feasible representation space, discarding transferable predictive factors that are not universally shared. To address this limitation, we propose subset-shared invariance, where predictive structure is assumed stable only within domain subsets. We implement this principle with a mixture-of-experts architecture, where each expert aligns the specific domains it serves and a routing mechanism composes subset-invariant components for prediction. This creates a routing-conditioned invariance, jointly learned with the representation. To facilitate effective decomposition, we develop training objectives that encourage selective alignment, confident and balanced routing, and diverse expert specialization. 
Experiments on DomainBed benchmarks demonstrate improved out-of-domain generalization and greater robustness under increasing domain heterogeneity. Our results suggest that DG should move beyond enforcing a single global invariance and instead model invariance through partially shared structure across domain subsets. 
\end{abstract}

\section{Introduction}

Domain generalization (DG) refers to the challenge of training models that perform well on unseen domains, without having access to data from these target domains during training \cite{Muandet2013DomainGV, Gulrajani2020InSO, koh2021wilds}. A widely used approach in DG is to develop domain-invariant representations by ensuring that the conditional distribution $P(Z \vert Y)$ remains consistent across different domains. This idea is central to many methods, including feature alignment \cite{Sun2016DeepCC, Peng2018MomentMF} and invariant learning frameworks such as invariant risk minimization (IRM) \cite{arjovsky2019invariant} and its extensions \cite{Ram2021FishrIG, Shi2021GradientMF}. Nonetheless, applying invariance uniformly to all domains may be too restrictive (see Figure~\ref{fig:intro_teaser}). In real-world settings, domains may only partially overlap in structure. For instance, specific features could be stable in some domains but vary in others \cite{Piratla2020EfficientDG, Yao2023ImprovingDG}. Therefore, imposing global invariance can eliminate useful predictive information, especially when domain shifts are heterogeneous \cite{Geirhos2021PartialSI, Wang2024lost}. This observation indicates that predictive structure is not universally shared, but instead distributed among subsets of domains. Similar patterns of subset-based heterogeneity are also found in multimodal learning, where individual samples may use different modality combinations \citep{yun2024flexmoe}.

To formalize this limitation, we analyze the effect of enforcing invariance across an increasing number of domains and show that it progressively reduces the mutual information between learned representations and labels. This reveals an inherent trade-off between invariance and predictive information. As invariance constraints become stronger, the representation becomes less informative for prediction. This observation suggests that DG should not enforce invariance globally, but should instead selectively preserve predictive structure that is stable only within subsets of domains.

We therefore argue that invariance in DG should be \emph{structured rather than global}. Specifically, we introduce \emph{subset-shared invariance}, where invariant structure is assumed to hold within subsets of domains rather than across all domains. This reframes DG as a structured decomposition problem. Instead of learning a single invariant representation, the model should capture multiple invariant components, each corresponding to a subset of domains, and combine them adaptively for prediction.

To instantiate this principle, we leverage a mixture-of-experts (MoE) architecture \cite{Jacobs1991AdaptiveMO, Shazeer2017OutrageouslyLN} as a mechanism for learning such a decomposition. A routing function induces a soft partition over inputs, which in turn defines subsets of domains associated with each expert. 
\begin{wrapfigure}{r}{0.45\textwidth} 
  \centering
  \includegraphics[width=\linewidth]{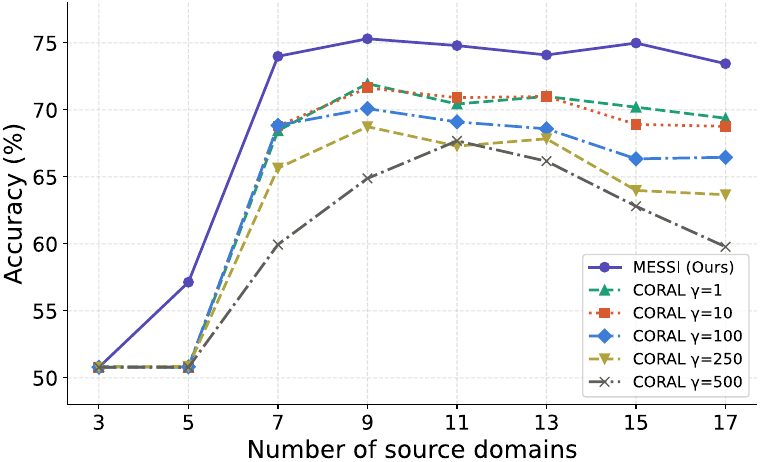}
  \caption{
  Test accuracy changes as the number of source domains increases under a fixed training budget. Adding more domains at first improves performance, but after a certain point, accuracy begins to decline. This pattern shows an important limitation of enforcing global invariance. When more heterogeneous domains are included, the shared invariant signal weakens and valuable predictive information is lost. In contrast, MESSI stays robust, implying that preserving subset-specific structure is crucial for generalization when domains differ significantly.
  }
  \label{fig:intro_teaser}
\end{wrapfigure}
We enforce invariance only within the subsets defined by routing, allowing the model to learn in which domains invariance is appropriate. This results in a \emph{routing-induced invariance}, where the model learns \emph{where} invariance should hold, rather than enforcing it uniformly across all domains.

MoE models have been used in DG to boost both capacity and specialization \cite{Li2022SparseMA, Chen2024LFMEAS, Dai2021GeneralizablePR}. However, earlier work relies on experts to act as feature extractors or domain-specific modules. Our method instead uses routing to directly shape how invariance is applied, so that expert specialization helps find invariant relationships specific to certain subsets of domains.

To operationalize this idea, we propose a subset-conditioned invariance objective that aligns class-conditional feature distributions only across domain pairs selected by routing. This alignment is implemented using an optimal transport formulation weighted by expert-specific routing mass, which induces soft domain assignments. To encourage meaningful decomposition, we introduce a diversity regularizer across experts, along with routing objectives that promote both confident (instance-level sparse) and balanced (dataset-level uniform) expert utilization.
We evaluate our approach on standard DG benchmarks and demonstrate consistent improvements over strong baselines. Further analysis shows that the learned routing patterns correspond to meaningful domain groupings, supporting the hypothesis that invariant structure is inherently subset-dependent rather than globally shared.
Our main contributions are as follows. 
\begin{itemize}
    \item We identify a fundamental limitation of global invariance in DG and show that enforcing invariance across increasing numbers of domains can reduce predictive information. Based on this insight, we introduce \emph{subset-shared invariance}, which reframes DG as learning structured, subset-dependent invariances rather than a single globally invariant representation.
    \item We propose a routing-based MoE framework that learns where invariance should hold. By using routing to induce a data-driven decomposition of domain relationships, our method enforces invariance selectively across domain subsets, enabling adaptive composition of invariant components.
    \item We demonstrate that the proposed approach improves out-of-domain generalization on standard DG benchmarks, particularly under heterogeneous domain shifts, and provides robustness as the number of source domains increases.
\end{itemize}

\section{Related Work}

\textbf{Domain Generalization and Invariant Representation Learning.} DG is commonly approached by learning representations that are invariant across domains, typically by enforcing the conditional distribution $P(Z \vert Y)$ to be consistent across environments \cite{Muandet2013DomainGV,Fang2013UnbiasedML,Li2018DeepDG}. This principle underlies a wide range of methods, including feature alignment approaches \cite{Sun2016DeepCC,Peng2018MomentMF, Li2018DomainGW} and invariant learning frameworks such as IRM and its variants \cite{arjovsky2019invariant,Ram2021FishrIG,Shi2021GradientMF}. Subsequent work has explored different mechanisms to approximate invariance, including data augmentation \cite{Shankar2018GeneralizingAD,Zhou2021DomainGW}, episodic and meta-learning strategies \cite{Balaji2018MetaRegTD,Dou2019DomainGV}, and self-supervised learning \cite{Carlucci2019DomainGB,lv2022causality}. More recent approaches refine invariance through causal reasoning \cite{Mahajan2020DomainGU,lv2022causality,Yin2024IntegratingMB,He2025LearningTC,Li2025TowardsSD}, disentangled representations \cite{Zhang2021TowardsPD,Cheng2024DisentangledPR,Pan2025MinimalSS}, and indirect alignment objectives \cite{Wei2025IndirectAA}. Despite their differences, these approaches largely share the assumption that a single representation can be made invariant across all domains. In contrast, we argue that invariant structure is often only partially shared, and should be modeled at the level of domain subsets.

\textbf{Optimization and Structural Constraints.} A complementary line of work studies DG from an optimization perspective, emphasizing that generalization depends not only on representation invariance but also on the solutions reached during training. Methods in this direction improve robustness through gradient regularization \citep{Ram2021FishrIG,Shi2021GradientMF,Nguyen2025FederatedDG,do2025domain}, geometry-aware optimization, and flatness constraints \cite{Cha2021SWADDG,Wang2023SharpnessAwareGM,Mansilla2021DomainGV}. More recent approaches explicitly shape training trajectories \cite{Ballas2025GradientGuidedAF,Cho2025OneStepGR,zhao2026unlearning} or constrain interference with pretrained knowledge \cite{Hu2024LearnTP,Yun2024SoMASV,Javed2024QTDoGQT,Wu2024DynamicST}. While these methods improve generalization through optimization dynamics, they do not explicitly address how invariant structure should be distributed across domains. Our work is complementary, focusing on the structure of invariance itself rather than the optimization process.

\textbf{Modular Experts and Parameter-Efficient Adaptation.} Recent DG work has moved toward modular architectures instead of a single shared model. Expert-based methods, such as sparse experts \cite{Li2022SparseMA,Chen2025PointMoELM} and multi-expert distillation \cite{Chen2024LFMEAS}, use specialized modules to capture transferable patterns. Parameter-efficient approaches such as prompt-based \cite{Wen2025DomainGI,Li2025GeneralizingVM,Li2024PromptDrivenDO}, adapter-based \cite{JiCustomizingDA}, low-rank \cite{Li2024FlatLoRALA}, multimodal \cite{Xu2026ReasoningDrivenML}, and regularization-based \cite{Cha2022DomainGB,Wei2023StrongerF,Cho2025PEERPM,Ni2024PACEMG} are focusing on targeted adaptation. While these methods improve robustness through specialization, they rarely model predictive factors shared only by subsets of domains or use routing for invariance. Recent MoE-based work, such as FlexMoE~\citep{yun2024flexmoe}, addresses subset-shared structure in multimodal settings with predefined indices, but requires explicit subset labels that are usually unavailable in DG.

\textbf{Difference from prior work.} Unlike prior expert-based or parameter-efficient DG methods, which primarily exploit specialization or constrained adaptation, our approach by contrast, models which domains should be aligned. We introduce \emph{subset-shared invariance} and use routing-conditioned expert alignment, so that experts capture predictive factors stable within only certain domain subsets. This avoids forcing all information into a single invariant representation. Unlike FlexMoE, which require predefined subset structures, our method discovers subset structure directly from data. The routing mechanism adaptively determines which domains to align, enabling data-driven decomposition of invariant structure without explicit labels for subset membership.

\section{Rethinking Feature Invariance in Domain Generalization}\label{sec:subset_shared}
We revisit the concept of feature invariance in DG. We show that enforcing invariance across an increasing number of domains progressively imposes stronger constraints on the learned representation, leading to a monotonic shrinkage of the invariant feature subspace. Consequently, globally invariant representations may discard predictive information that is shared only across subsets of domains.

\subsection{Problem Setup}
We consider the multi-source DG setting with $K$ source domains and no access to target-domain data during training.
Formally, let $X \in \mathcal{X}$ denote the input, $Y \in \mathcal{Y}$ the label, and $D \in \{1,\dots,K\}$ the domain index, where $\{\mathcal{D}_k\}_{k=1}^{K}$ are the source-domain distributions over $(X,Y)$.
A predictor is parameterized as $\hat y = f_\theta(z)$ with representation $z = g_\phi(x)$, where $g_\phi : \mathcal{X} \to \mathbb{R}^d$ is a feature extractor and $f_\theta$ is a classifier.
Training then minimizes the empirical risk across source domains \citep{Muandet2013DomainGV,Gulrajani2020InSO,Vapnik1991PrinciplesOR}:
\begin{align}
\mathcal{L}_{\mathrm{cls}}
=
\frac{1}{K}\sum_{k=1}^{K}
\mathbb{E}_{(x,y)\sim\mathcal{D}_k}
\big[\ell(f_\theta(g_\phi(x)),y)\big].
\label{eq:cls_loss}
\end{align}
\subsection{Monotonic Shrinkage in Globally Learned Invariance}
To analyze the multi-DG problem, we first reformulate the general DG optimization problem as follows. Let $Z = g_\phi(X)$ denote the learned representation. For a subset of source domains
$\mathcal{K} \subseteq \{1,\dots,K\}$, define the pairwise index set $\mathcal{P}(\mathcal{K})
= \{(i,j)\in \mathcal{K}\times\mathcal{K} : i<j\}.$
We consider the following information-theoretic objective:
\begin{align}
\max_{\phi}~\mathcal{J}_{\mathcal{K}}(\phi)
=
I(Z;Y)
-
\lambda
\sum_{(i,j)\in \mathcal{P}(\mathcal{K})}
I\!\left(Z;D
\,\middle|\,
Y, D\in\{i,j\}
\right),
\qquad \lambda > 0.
\label{eq:pairwise_cmi_objective}
\end{align}
Theoretically, the first term encourages label-predictive representations, while the second term penalizes domain-discriminative information within each domain pair after conditioning on the label.
We denote the optimal value by  $V_{\mathcal{K}} = \sup_{\phi}\, \mathcal{J}_{\mathcal{K}}(\phi)$. 
From the above objective function, we have the following propositions.
\begin{proposition}\label{prop:domain-expansion}
Let $\mathcal{K}\subseteq \mathcal{K}'$. Then, for every $\phi$, we have $\mathcal{J}_{\mathcal{K}'}(\phi)
\le \mathcal{J}_{\mathcal{K}}(\phi).$
Consequently, the following inequality holds: $V_{\mathcal{K}'} \le V_{\mathcal{K}}$.
\end{proposition}

\begin{proposition}\label{prop:pairwise-invariance}
For any $(i,j)\in\mathcal{P}(\mathcal{K})$, $I\!\left(Z;D \,\middle|\, Y,\; D\in\{i,j\} \right) = 0$ if and only if the conditional distributions satisfy $P(Z\vert Y,D=i) = P(Z\vert Y,D=j)$.
\end{proposition}

Proposition~\ref{prop:domain-expansion} formalizes that enlarging the domain set introduces additional nonnegative pairwise invariance penalties. As a result, the objective decreases pointwise for every $\phi$, and hence the optimal value is non-increasing. Proposition~\ref{prop:pairwise-invariance} shows that each pairwise term vanishes if and only if $P(Z\vert Y,D=i)=P(Z\vert Y,D=j)$. Enforcing this over all $(i,j)\in\mathcal{P}(\mathcal{K})$ yields $Z \perp D \vert Y$, i.e., global conditional invariance across domains. From the Propositions~\ref{prop:domain-expansion} and \ref{prop:pairwise-invariance}, we have the following theorems.

\begin{theorem}\label{thm:shrinking-set}
Assume that, for a domain set $\mathcal{K}$, we have
$I\!\left(
Z;D
\,\middle|\,
Y,\; D\in\{i,j\}
\right)
=
0
,~\forall (i,j)\in\mathcal{P}(\mathcal{K}).$
Then, we have
\begin{align}
P(Z\vert Y,D=i)
=
P(Z\vert Y,D=j)
\qquad
\forall i,j\in\mathcal{K},
\label{eq:all_pairs_equal}
\end{align}
and thus, $Z \perp D \vert Y$ over $D\in\mathcal{K}$.
\end{theorem}

\begin{theorem}\label{thm:monotonicity}
Suppose that for every domain set $\mathcal{K}$, there exists an optimizer
$\phi^{*}_{\mathcal{K}}$ of \eqref{eq:pairwise_cmi_objective} satisfying exact pairwise
conditional invariance:
\begin{align}
I\!\left(
Z^{*}_{\mathcal{K}};D
\,\middle|\,
Y,\; D\in\{i,j\}
\right)
=
0
\qquad
\forall (i,j)\in\mathcal{P}(\mathcal{K}),
\label{eq:exact_invariance_opt}
\end{align}
where $Z^{*}_{\mathcal{K}}=g_{\phi^{*}_{\mathcal{K}}}(X)$. Then, for
$\mathcal{K}\subseteq \mathcal{K}'$, we have $I(Z^{*}_{\mathcal{K}'};Y) \le I(Z^{*}_{\mathcal{K}};Y).$
\end{theorem}
Theorem~\ref{thm:shrinking-set} shows that each pairwise term enforces equality of conditional distributions across domains, and that vanishing penalties over all pairs imply global invariance $Z \perp D \vert Y$. Theorem~\ref{thm:monotonicity} further shows that enlarging the domain set introduces additional invariance constraints, reducing the optimal objective value and, under exact invariance, the achievable label information $I(Z;Y)$. 

\textbf{Key takeaway.}
The above analysis reveals an inherent trade-off between domain invariance and predictive information. For a moderate number of domains, the optimal representation $Z^{*}_{\mathcal{K}}$ can retain sufficient label-relevant information while satisfying invariance constraints, i.e., $I(Z^{*}_{\mathcal{K}};Y)$ remains above the level required for accurate prediction. However, as the number of domains increases, the invariance constraints become progressively stronger, leading to a monotonic decrease in $I(Z^{*}_{\mathcal{K}};Y)$. When this quantity falls below the level necessary for reliable prediction, the representation becomes overly invariant and loses discriminative information, resulting in degraded generalization performance. Empirically, we observe a non-monotonic pattern under the fixed-budget domain-growth setting (see Appendix~\ref{app:rotated_colored_mnist_results}). Specifically, increasing $K$ initially improves target accuracy, but beyond a certain point leads to performance degradation. This behavior is consistent with the predicted trade-off: moderate domain diversity encourages the suppression of spurious correlations and improves invariance, while excessive heterogeneity enforces overly strong invariance constraints that suppress label-relevant information.

\subsection{Subset-Conditioned Invariance}

The analysis above suggests that enforcing invariance uniformly across all domains may be overly restrictive when predictive structure is only partially shared. We therefore consider a structured formulation in which invariance is conditioned on latent subsets of domains.

\begin{definition}[Subset-Conditioned Invariance]\label{def:subset-conditioned}
Let $S \in \{1,\dots,M\}$ denote a latent subset variable and let $Z^{(m)} = g^{(m)}_\phi(X)$ be the representation associated with subset $m$. We say that $Z^{(m)}$ satisfies \emph{subset-conditioned invariance} if $I(Z^{(m)}; D \vert Y, S=m)=0$.
Equivalently, conditioned on the label and subset assignment, the representation is invariant to domain-specific variation, i.e., $P(Z^{(m)} \vert Y, D=i, S=m)
=
P(Z^{(m)} \vert Y, D=j, S=m),$
for all domains $i,j$ satisfying
$P(D=i \vert S=m)>0$ and $P(D=j \vert S=m)>0$.
\end{definition}
Given Definition~\ref{def:subset-conditioned}, we consider the following objective:
\begin{align}
\max_{\phi, \pi} \sum_{m=1}^{M} I(Z^{(m)}; Y\vert S=m), 
\quad \mathrm{s.t.} ~~~ I(Z^{(m)}; D \vert Y, S=m) = 0, \quad \forall m. \notag
\end{align}
This formulation encourages each component $Z^{(m)}$ to capture predictive features that are invariant within a subset of domains.
\begin{lemma}[Subset-Conditioned Invariance]\label{lemma:subset}
For a fixed subset $m$, the condition $I(Z^{(m)}; D \vert Y, S=m) = 0$ holds if and only if, for all domains $i,j$ such that $P(D=i \vert S=m) > 0$ and $P(D=j \vert S=m) > 0$, $P(Z^{(m)} \vert Y, D=i, S=m) = P(Z^{(m)} \vert Y, D=j, S=m).$
\end{lemma}
\begin{corollary}
    Subset-conditioned invariance can be enforced through pairwise alignment across domains within each subset.
\end{corollary}
\begin{theorem}[Optimality of Subset-Conditioned Invariance]\label{thm:subset-invariance}
Suppose there exist subsets $\{\mathcal{D}_m\}_{m=1}^{M}$ such that for each subset $m$, $P(Y \vert X, D=i) = P(Y \vert X, D=j), \quad \forall i,j \in \mathcal{D}_m$, and the predictive structures differ across subsets. Then any representation satisfying global invariance $I(Z; D \vert Y) = 0$ may discard predictive information, whereas a subset-conditioned representation $\{Z^{(m)}\}$ satisfying
$I(Z^{(m)}; D \vert Y, S=m) = 0$ for each $m$ can preserve all subset-specific predictive factors.
\end{theorem}
The above results show that enforcing $I(Z^{(m)}; D \vert Y, S=m)=0$ is equivalent to aligning class-conditional feature distributions across domain pairs within each subset. This provides a practical route to approximate the objective via pairwise alignment weighted by subset assignment probabilities.
In the next section, we develop a method that instantiates this formulation using a MoE architecture, where experts represent subset-invariant components and routing approximates the latent subset. Proofs of all propositions, lemmas, and theorems are provided in Appendix~\ref{app:proofs}.

\section{Proposed Method}

To avoid the over-constraint induced by global invariance, we propose a routing-based MoE framework that learns multiple subset-invariant components and composes them adaptively for prediction (Figure~\ref{fig:method_overview}). Motivated by the subset-conditioned invariance formulation in Section~\ref{sec:subset_shared}, each expert is designed to capture a distinct subset-specific predictive mechanism, while the routing function approximates the latent subset assignment variable $S$ through input-dependent expert selection. Given an input $x$, a shared encoder produces a feature representation $u=b_\phi(x)$, the router outputs a distribution over experts $\pi(x)=\mathrm{softmax}(g_\theta(u))$, and each expert produces an expert-specific representation $z^{(m)}=h_m(u)$. The final representation is formed by aggregating expert outputs according to the routing distribution $z(x)=\sum_{m=1}^M \pi_m(x) z^{(m)}$. This allows different subset-invariant components to be activated adaptively for different samples.

This design introduces three challenges: (1) determining which domain pairs should be aligned for each expert, (2) preventing experts from collapsing to redundant representations, and (3) ensuring routing remains both confident at the instance level and balanced across the dataset. To address these challenges, we optimize the following objective:
\begin{align}
\mathcal L
=
\mathcal L_{\mathtt{cls}}
+
\lambda_{\mathtt{ssi}}\mathcal L_{\mathtt{ssi}}^{\mathrm{OT}}
+
\lambda_{\mathrm{sp}}\mathcal L_{\mathtt{sp}}
+
\lambda_{\mathrm{bal}}\mathcal L_{\mathtt{bal}}
+
\lambda_{\mathrm{div}}\mathcal L_{\mathtt{div}}.
\end{align}
Here, $\mathcal L_{\mathrm{cls}}$ denotes the classification loss, $\mathcal L^{\mathrm{OT}}_{\mathtt{ssi}}$ enforces routing-conditioned subset alignment, $\mathcal L_{\mathtt{sp}}$ encourages confident routing, $\mathcal L_{\mathtt{bal}}$ promotes balanced expert utilization, and $\mathcal L_{\mathtt{div}}$ encourages diverse expert representations. We describe each component below. Architectural details are provided in Appendix~\ref{sec:messi-architecture}. Full optimization details and hyperparameters are provided in Appendices~\ref{app:training} and~\ref{app:hyperparameters}.
\begin{figure*}[t]
    \centering
    \includegraphics[width=\textwidth]{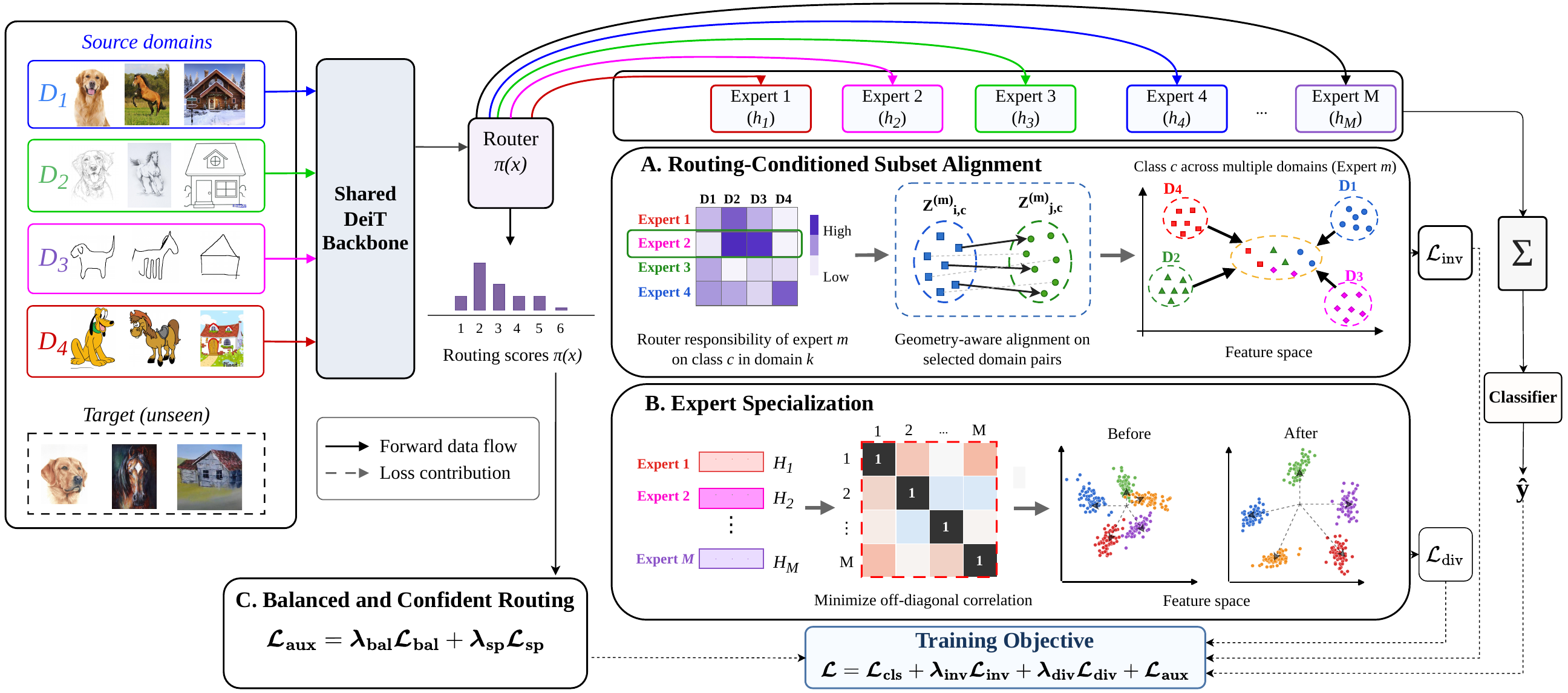}
    \caption{Overview of MESSI. An input is encoded by a shared backbone and routed to multiple experts,
    whose outputs are composed into the final representation for prediction. The training objective combines
    classification loss with subset-conditioned invariance, confident and balanced routing, and expert specialization terms.}
    \label{fig:method_overview}
    \vspace{-0.6cm}
\end{figure*}

\subsection{Routing-Conditioned Subset Alignment}\label{sec:subset}

Section~\ref{sec:subset_shared} shows that enforcing conditional invariance uniformly across all domain pairs can become overly restrictive as the number and heterogeneity of source domains increase. In particular, \eqref{eq:pairwise_cmi_objective} penalizes domain-discriminative information across all pairs $(i,j)$, which leads to progressively stronger constraints as the domain set grows. To avoid this effect, we relax the objective by enforcing invariance selectively over a subset of domain pairs.
Specifically, we enforce invariance at the level of expert-specific representations. For each expert $m$, domain $k$, and class $c$, we compute the average routing responsibility
$\rho^{(m)}_{k,c} = \frac{1}{|\mathcal D_{k,c}|} \sum_{(x,y)\in\mathcal D_k,\, y=c} \pi_m(x)$,
which measures how strongly expert $m$ is associated with class $c$ in domain $k$.
Using these routing statistics, we define pairwise alignment weights
$a^{(m)}_{ijc}=\sigma(\alpha \rho^{(m)}_{i,c})\cdot\sigma(\alpha \rho^{(m)}_{j,c})$,
which become large when expert $m$ assigns high routing mass to class $c$ in both domains $i$ and $j$.
Consequently, the invariance objective is:
\begin{align}
&\mathcal L_{\mathtt{ssi}}^{\mathrm{OT}}
=
\sum_{m=1}^{M}\sum_{c\in\mathcal Y}\sum_{i<j}
a^{(m)}_{ijc}\,
W_{\varepsilon}\!\left(\mathcal Z^{(m)}_{i,c}, \mathcal Z^{(m)}_{j,c}\right), \\
\mathrm{s.t.} ~~~ 
&W_{\varepsilon}(P,Q)
=
\min_{\gamma\in\Pi(P,Q)}
\mathbb E\|z-\tilde z\|_2^2
+
\varepsilon\,\mathrm{KL}(\gamma \,\|\, P\otimes Q). \notag 
\end{align}
Our primary formulation uses entropic OT as the discrepancy measure (the motivation is demonstrated in Appendix~\ref{app:subset-alignment}). An MMD-based variant that uses the same routing-conditioned pair-selection mechanism is described in Appendix~\ref{sec:ssi_ot_scalability}.


\subsection{Expert Specialization}\label{sec:specialization}
A final failure mode is redundancy across experts: even with well-behaved routing, different experts may
still converge to similar outputs. 
To discourage this, let \(H_m\in\mathbb R^{B\times r}\) be the batch output
matrix of expert \(m\), and let \(\tilde H_m=H_m/(\|H_m\|_F+\varepsilon)\) be its normalized version.
Specifically, we minimize
\begin{align}
\mathcal L_{\mathtt{div}}
=
\sum_{m\neq n}
\left\|
\frac{1}{B}\tilde H_m^\top \tilde H_n
\right\|_F^2.
\label{eq:div}
\end{align}








The matrix \(\frac{1}{B}\tilde H_m^\top \tilde H_n\) measures the correlation between the feature directions of
experts \(m\) and \(n\) over the minibatch. If two experts produce similar representations, this matrix has
large magnitude and \eqref{eq:div} increases. Minimizing \(\mathcal L_{\mathtt{div}}\) therefore drives
experts toward complementary predictive directions rather than redundant ones.


Recent works have explored enforcing orthogonality among experts to promote functional diversity \citep{Hu2024LearnTP, hendawy2024multitask, oldfield2024multilinear, feng2025omoe, Guo2025AdvancingES}. Methods such as MOORE \citep{hendawy2024multitask} and OMOE \citep{feng2025omoe} project expert outputs into an orthogonal space via Gram-Schmidt, but this may not preserve task-relevant information and can discard salient features. In contrast, the orthogonality loss in \citep{Guo2025AdvancingES} admits a trivial minimum where expert outputs collapse to zero, reducing the penalty without ensuring meaningful diversity. Our specialization loss instead promotes diversity by minimizing cosine similarity between expert outputs, encouraging directional decorrelation while avoiding scale-related degeneracies.

\subsection{Balanced and Confident Routing}\label{sec:balance}
The subset-conditioned invariance loss determines \emph{which} class-conditional domain pairs each expert should align, but it does not by itself guarantee that routing produces a meaningful decomposition. In particular, two degenerate regimes can arise. If routing is too diffuse, each input weakly activates many experts, so the representation remains effectively shared and expert specialization does not emerge. If routing collapses globally, only a small subset of experts receives most of the traffic, while the remaining experts are rarely trained and therefore cannot capture distinct subset-shared factors. We address these two failure modes by regularizing routing at both the instance and dataset levels:
\begin{align}
\mathcal L_{\mathtt{sp}}
=
- \mathbb E_x\!\left[\sum_{m=1}^M \pi_m(x)\log \pi_m(x)\right],
\quad 
\mathcal L_{\mathtt{bal}}
=
\sum_{m=1}^M\left(\mathbb E_x[\pi_m(x)]-\frac{1}{M}\right)^2.
\end{align}
The two objectives regularize routing at complementary scales. The sparsity term $\mathcal L_{\mathtt{sp}}$ acts locally on each routing distribution $\pi(x)$, encouraging low-entropy expert assignment for individual samples. The balancing term $\mathcal L_{\mathtt{bal}}$ instead operates on the aggregated routing statistics $\mathbb E_x[\pi_m(x)]$, preventing a small subset of experts from dominating across the dataset.

\section{Experiments}
\label{sec:experiments}

Our experiments are organized around three claims. $\bm{\mathcal{C}}_1$: Subset-conditioned expert alignment improves standard DG performance. $\bm{\mathcal{C}}_2$: \textsc{MESSI} is more robust than globally aligned methods as the number of source domains increases. $\bm{\mathcal{C}}_3$: The gains arise from routing-induced subset-conditioned alignment rather than additional MoE capacity, global alignment, or random sparse alignment.


\begin{table}[!h]
\caption{%
OOD accuracy (\%) under the DomainBed training-domain validation criterion, averaged over three seeds. \textbf{MESSI-Ti} and \textbf{MESSI-S} use DeiT-Ti/16 and DeiT-S/16 backbones, respectively. Best results are in \textbf{bold}; second-best results are \underline{underlined}.
}
\label{tab:main}
\centering
\small
\setlength{\tabcolsep}{4.5pt}
\renewcommand{\arraystretch}{1.05}
\begin{tabular*}{\textwidth}{@{\extracolsep{\fill}}llccccc}
\toprule
Method & Backbone & PACS $\uparrow$ & OfficeHome $\uparrow$ & TerraInc $\uparrow$ & DomainNet $\uparrow$ & Avg. $\uparrow$ \\
\midrule
ERM~\citep{Vapnik1991PrinciplesOR}
  & ResNet-50
  & 83.8\ms{0.8} & 66.6\ms{0.4} & 47.2\ms{0.8} & 41.6\ms{0.2} & 59.8 \\
CORAL~\citep{Sun2016DeepCC}
  & ResNet-50
  & 83.5\ms{0.4} & 66.0\ms{0.5} & 43.8\ms{0.9} & 39.2\ms{0.3} & 58.1 \\
Fishr~\citep{Ram2021FishrIG}
  & ResNet-50
  & 85.5\ms{0.2} & 68.6\ms{0.2} & 47.4\ms{1.6} & 41.7\ms{0.3} & 60.8 \\
SAGM~\citep{Wang2023SharpnessAwareGM}
  & ResNet-50
  & 86.4\ms{1.2} & 69.4\ms{0.2} & 48.8\ms{1.0} & 43.2\ms{0.4} & 61.9 \\
LFME~\citep{Chen2024LFMEAS}
  & ResNet-50
  & 84.9\ms{0.4} & 68.5\ms{0.2} & \textbf{49.5}\ms{\textbf{0.8}} & 38.6\ms{0.2} & 60.4 \\
\midrule
ERM~\citep{Vapnik1991PrinciplesOR}
  & DeiT-S/16
  & 86.2\ms{0.1} & 72.2\ms{0.4} & 42.0\ms{0.8} & 47.3\ms{0.2} & 61.9 \\
DynMoE~\citep{Guo2024DynamicMO}
  & DeiT-S/16
  & 85.2\ms{0.4} & 73.4\ms{0.3} & 44.5\ms{0.6} & 45.9\ms{0.4} & 62.3 \\
GMoE~\citep{Li2022SparseMA}
  & DeiT-S/16
  & 87.3\ms{0.1} & 73.5\ms{0.1} & 48.2\ms{0.5} & \underline{47.8\ms{0.5}} & 64.2 \\
OMoE~\citep{feng2025omoe}
  & DeiT-S/16
  & 87.1\ms{0.3} & 73.2\ms{0.1} & 46.0\ms{0.5} & 46.1\ms{0.2} & 63.1 \\
\midrule
\rowcolor{gray!10}
\textbf{MESSI-Ti-OT (Ours)}
  & DeiT-Ti/16
  & 86.0\ms{0.3} & 69.6\ms{0.2} & 42.5\ms{1.1} & 43.3\ms{0.4} & 60.4 \\

\rowcolor{gray!10}
\textbf{MESSI-S-MMD (Ours)}
  & DeiT-S/16
  & \underline{89.3\ms{0.6}} & \underline{74.0\ms{0.3}} & 47.5\ms{0.4} & 46.2\ms{0.4} & \underline{64.3} \\

\rowcolor{gray!10}
\textbf{MESSI-S-OT (Ours)}
  & DeiT-S/16
  & \textbf{90.9}\ms{\textbf{0.7}} & \textbf{76.1}\ms{\textbf{0.2}} & \underline{49.3\ms{0.2}} & \textbf{48.6}\ms{\textbf{0.3}} & \textbf{66.2} \\
\bottomrule
\end{tabular*}
\end{table}
\subsection{Experimental Setup}
\label{sec:setup}

\textbf{Benchmarks.} We evaluate standard OOD generalization on DomainBed~\citep{Gulrajani2020InSO},
including PACS~\citep{Li2017DeeperBA}, OfficeHome~\citep{Venkateswara2017DeepHN},
TerraIncognita~\citep{Beery2018RecognitionIT}, and DomainNet~\citep{Peng2018MomentMF}, and test domain growth on Rotated-Colored MNIST (see Appendix~\ref{app:rotated_colored_mnist} for details), a controlled variant inspired by Rotated MNIST and Colored  MNIST~\citep{ghifary2015domain, arjovsky2019invariant}. We follow the DomainBed leave-one-domain-out protocol described in Appendix~\ref{app:domainbed_protocol}, while Rotated-Colored MNIST allows us to vary the number of source domains to validate the effect of monotonic shrinkage. Details of Rotated-Colored MNIST and the fixed-budget domain-growth protocol are provided in Appendices~\ref{app:rotated_colored_mnist} and~\ref{app:domain_growth_protocol}.

\textbf{Implementation.}
MESSI uses a pretrained DeiT \citep{touvron2021training} encoder followed by an $M{=}6$ expert MoE head, where each expert is a two-layer MLP and the router computes soft expert weights from the final CLS feature. We evaluate both DeiT-Ti/16 and DeiT-S/16 backbones under the same input resolution, augmentations, optimization schedule, and model-selection criterion as the baselines. All experiments are run on NVIDIA H100 GPUs. 
Baselines and parameter counts are provided in Appendix~\ref{app:implementation}.

\subsection{DomainBed Results}
\label{sec:main_results}
Table~\ref{tab:main} evaluates $\bm{\mathcal{C}}_1$ by reporting OOD accuracy on DomainBed benchmarks. Reported baselines provide broader DG context, while \textit{MESSI-S} enables a same-backbone comparison to transformer-based methods. Additional backbone comparisons are presented in Appendix~\ref{app:domainBed_full}. The same-backbone comparison in Table~\ref{tab:main} and the model-size analysis in Appendix~\ref{app:model_size} rule out the simple explanation that MESSI improves by using a larger MoE model. MESSI-S uses fewer method-specific parameters and lower inference cost than GMoE/OMoE, while achieving higher average OOD accuracy. However, parameter count alone does not isolate the training mechanism. 
\begin{wraptable}{r}{0.57\textwidth}
\caption{
Fixed-budget domain-growth results. Peak is the best test accuracy across source-domain budgets, and
$\mathrm{Drop}=\mathrm{Peak}-\mathrm{Acc}(K_{\max})$. Best results are shown in \textbf{bold}. Full curves and setup details are provided in Appendix~\ref{app:rotated_colored_mnist_results}.
}
\label{tab:domain_growth}
\centering
\small
\setlength{\tabcolsep}{4pt}
\begin{tabular}{lcccc}
\toprule
Method & Peak $\uparrow$ & $\mathrm{Acc}(K_{\max})$ $\uparrow$ & Drop $\downarrow$ & Avg $\uparrow$ \\
\midrule
CORAL ($\gamma{=}1$)   & 71.96 & 69.38 & 2.58 & 65.38 \\
CORAL ($\gamma{=}10$)  & 71.64 & 68.77 & 2.87 & 65.20 \\
CORAL ($\gamma{=}100$) & 70.10 & 66.46 & 3.64 & 63.88 \\
CORAL ($\gamma{=}250$) & 68.75 & 63.67 & 5.08 & 62.36 \\
CORAL ($\gamma{=}500$) & 67.69 & 59.78 & 7.91 & 60.62 
\\ 

\rowcolor{gray!10}
\textbf{MESSI (Ours)} & \textbf{75.31} & \textbf{73.45} & \textbf{1.86} & \textbf{68.32} \\
\bottomrule
\end{tabular}
\end{wraptable}
We therefore further compare against same-architecture controls in Appendix~\ref{app:objective_ablations}, including a classification-only MoE and variants that remove subset-conditioned alignment, routing sparsity, load balancing, or diversity. 
These controls test whether the gain remains after model capacity, backbone, routing module, and training budget are fixed. Per-target-domain accuracies for PACS, OfficeHome, TerraIncognita, and DomainNet are reported in Appendix~\ref{app:domainbed_per_domain}. As a supporting many-domain evaluation, we report WILDS-iWildCam \citep{koh2021wilds, beery2021iwildcam} results in Appendix~\ref{app:iwildcam_results}, using the protocol described in Appendix~\ref{app:iwildcam_protocol}.


\subsection{In-depth Analysis}
\label{sec:analysis_summary}
\textbf{MESSI is more robust under domain growth.}
Table~\ref{tab:domain_growth} evaluates $\bm{\mathcal{C}}_2$ on Rotated-Colored MNIST under a fixed-budget domain-growth protocol. The target domain and total source training budget are fixed, while the number of source domains increases. This setting separates the effect of increasing domain heterogeneity from the effect of adding more training data. MESSI achieves the best peak accuracy, the best accuracy at \(K_{\max}\), and the smallest drop from peak accuracy. The strongest CORAL variant reaches \(71.96\%\) peak accuracy and drops by \(2.58\) points at \(K_{\max}\), whereas MESSI reaches \(75.31\%\) and drops by only \(1.86\) points. This supports $\bm{\mathcal{C}}_2$: routing-conditioned subset alignment is less affected by the over-constraint induced by global moment alignment when source-domain heterogeneity grows. The fixed-budget protocol and full curves are provided in Appendices~\ref{app:domain_growth_protocol}, and~\ref{app:rotated_colored_mnist_results}.
\begin{wraptable}{r}{0.46\textwidth}
\centering
\caption{ $\bm{\mathcal{C}}_3$ \textbf{diagnostic on PACS.}}
\label{tab:c2_main}
\small
\setlength{\tabcolsep}{3.2pt}
\renewcommand{\arraystretch}{1.08}
\begin{tabular}{lcc}
\toprule
Metric & Control & MESSI \\
\midrule
Resp. corr. \(r\) & \(+1.00\) & \(-0.20\) \\
Slot overlap \(J\) & \(0.95\) & \(0.04\) \\
Discrepancy \(\downarrow\) & \(0.234\pm0.020\) & \(\mathbf{0.096\pm0.014}\) \\
\bottomrule
\end{tabular}
\end{wraptable}

\textbf{Routing-conditioned subset alignment explains the gain.}
Table~\ref{tab:mechanism_ablation_main} evaluates $\bm{\mathcal{C}}_3$ using same-architecture controls on PACS.
Both global and random sparse alignment reduce accuracy relative to ERM-MoE, showing that simply adding alignment to an MoE architecture is insufficient. Domain-only routing also underperforms, indicating that routing mass alone is too coarse without class-domain responsibility. In contrast, the routing-conditioned class-domain selector avoids the degradation observed with global and random alignment, while the full objective achieves the best performance. These results support $\bm{\mathcal{C}}_3$: the improvement depends on where alignment is applied, whereas sparse and balanced routing make the subset decomposition more effective.

\begin{table}[h]
\caption{
Mechanism ablation under the same MoE architecture. The first group isolates the alignment selector. The last two rows separate routing-conditioned class-domain alignment from routing regularization. Best results are shown in \textbf{bold}.
}
\label{tab:mechanism_ablation_main}
\centering
\footnotesize
\setlength{\tabcolsep}{3.2pt}
\renewcommand{\arraystretch}{1.05}
\begin{tabular*}{\textwidth}{@{\extracolsep{\fill}}llccc@{}}
\toprule
Variant
& Alignment selector
& Class-aware
& Routing-cond.
& PACS $\uparrow$ \\
\midrule
ERM-MoE
& none
& --
& --
& 89.1\ms{0.3} \\

Global-MoE
& all class-domain slots
& \checkmark
& --
& 87.1\ms{0.2} \\

Random-Subset-MoE
& matched random slots
& \checkmark
& --
& 87.0\ms{0.1} \\

Domain-Only Routing
& routing domain mass
& --
& \checkmark
& 88.2\ms{0.2} \\

MESSI w/o \( \mathcal{L}_{\mathtt{sp}},\mathcal{L}_{\mathtt{bal}}\)
& routing class-domain mass
& \checkmark
& \checkmark
& 89.0\ms{0.2} \\

\rowcolor{gray!10}
\textbf{MESSI (Ours)}
& routing class-domain mass
& \checkmark
& \checkmark
& \textbf{90.9}\ms{\textbf{0.7}} \\
\bottomrule
\end{tabular*}
\end{table}

\textbf{MESSI changes the alignment structure.}
We further test whether MESSI merely reweights a global alignment pattern or learns a different subset structure. A slot is a class-domain-expert tuple \((m,i,j,c)\), indicating that expert \(m\) aligns class \(c\) between domains \(i\) and \(j\). We compare selectors using the Pearson correlation \(r\) between flattened slot weights and the Jaccard overlap \(J\) between top-selected slot sets. Higher values indicate similar alignment structures.

Table~\ref{tab:c2_main} shows that matched random sparse selection remains close to global alignment, with \(r=1.00\) and \(J=0.95\). MESSI instead selects a substantially different set of slots, with \(r=-0.20\) and \(J=0.04\). On the same MESSI-selected slots \(S_{\textsc{M}}\), MESSI also reduces class-conditional discrepancy from \(0.234\pm0.020\) to \(0.096\pm0.014\). Thus, the gain is not explained by MoE capacity or sparsity alone; MESSI learns a different alignment structure and aligns those selected slots more effectively.

\textbf{Ablation test.} Objective ablations are reported in Appendix~\ref{app:objective_ablations}. Additional diagnostics on expert specialization, routing behavior, and routing-aware alignment controls are provided in Appendices~\ref{app:expert_specialization}, \ref{app:routing_diagnostics}, and~\ref{app:routing_aware_alignment}. Appendix~\ref{app:lambda_sweep} studies the pairwise-to-global invariance trade-off by sweeping the invariance weight and measuring pairwise discrepancy, domain predictability, and target accuracy.

\paragraph{Training Cost and Inference Efficiency}
\label{sec:training_time}

\begin{wraptable}{r}{0.60\textwidth}
    \centering
    \footnotesize
    \vspace{-10pt}
    \caption{
    Training and inference cost on PACS, measured on NVIDIA H100 GPUs.
    Train time is measured in wall-clock seconds per 1K optimization
    steps after a 100-step warm-up. Relative time is normalized by ERM.
    Inference is measured with all training-only losses disabled.
    }
    \label{tab:training_time}
    \setlength{\tabcolsep}{3pt}
    \renewcommand{\arraystretch}{1.02}
    \begin{tabular}{llccc}
        \toprule
        Method
        & Backbone
        & Train
        & Rel.
        & Infer. \\
        \midrule
        ERM
        & ResNet-50
        & 103.78$\pm$0.50
        & 1.00$\times$
        & 32.42$\pm$0.01 \\

        SAGM
        & ResNet-50
        & 211.37$\pm$1.26
        & 2.04$\times$
        & 32.44$\pm$0.01 \\

        GMoE
        & DeiT-S/16
        & 209.71$\pm$0.49
        & 2.02$\times$
        & 74.22$\pm$0.11 \\

        GMoE+SAGM
        & DeiT-S/16
        & 429.88$\pm$15.86
        & 4.14$\times$
        & 74.17$\pm$0.08 \\

        \rowcolor{gray!10}
        \textbf{MESSI-Ti-OT}
        & DeiT-Ti/16
        & 1200.86$\pm$22.38
        & 11.57$\times$
        & \textbf{23.50$\pm$0.01} \\

        \rowcolor{gray!10}
        \textbf{MESSI-S-MMD}
        & DeiT-S/16
        & 439.68$\pm$5.35
        & 4.24$\times$
        & 61.40$\pm$0.01 \\

        \rowcolor{gray!10}
        \textbf{MESSI-S-OT}
        & DeiT-S/16
        & 1286.26$\pm$28.57
        & 12.39$\times$
        & 61.41$\pm$0.01 \\
        \bottomrule
    \end{tabular}
    \vspace{-10pt}
\end{wraptable}

Table~\ref{tab:training_time} shows that MESSI introduces training-time overhead
mainly through the subset-conditioned discrepancy. MESSI-S-MMD costs
\(4.24\times\) ERM and is close to GMoE+SAGM, while MESSI-S-OT costs
\(12.39\times\) ERM. The gap between the two MESSI-S variants isolates the cost
of the entropic OT solver. MMD gives a non-iterative discrepancy, whereas OT
runs Sinkhorn iterations over selected class-conditional domain pairs.


\section{Conclusion}
\label{sec:conclusion}
We proposed MESSI, a domain generalization framework based on subset-shared invariance. Our analysis shows that enforcing invariance across many domains can restrict the feasible predictive representation and discard factors that are stable only across domain subsets. MESSI addresses this issue with a routing-conditioned MoE model, where each expert aligns the class-conditional domain pairs it actively serves and the routed mixture composes the resulting subset-invariant components for prediction. Experiments on DomainBed and controlled source-domain expansion support the claim that DG methods should preserve and compose subset-level invariances rather than always enforce a single invariant space across all domains. We discuss failure modes, including reliance on reliable routing specialization and additional alignment cost, in Appendix~\ref{sec:limitations}.

\clearpage
\bibliographystyle{plainnat}
\bibliography{ref}

\clearpage
\appendix

\section{Proofs}
\label{app:proofs}

\subsection{Proof of Proposition~\ref{prop:domain-expansion}}
\label{app:proof_domain_expansion}

\begin{proof}
Since $\mathcal{K}\subseteq \mathcal{K}'$, we have
\begin{align}
    \mathcal{P}(\mathcal{K})\subseteq \mathcal{P}(\mathcal{K}').
\end{align}
Therefore,
\begin{align}
&\sum_{(i,j)\in\mathcal{P}(\mathcal{K}')}
I\!\left(Z;D\,\middle|\,Y,D\in\{i,j\}\right) \nonumber\\
&=
\sum_{(i,j)\in\mathcal{P}(\mathcal{K})}
I\!\left(Z;D\,\middle|\,Y,D\in\{i,j\}\right)
+
\sum_{(i,j)\in\mathcal{P}(\mathcal{K}')\setminus\mathcal{P}(\mathcal{K})}
I\!\left(Z;D\,\middle|\,Y,D\in\{i,j\}\right).
\end{align}
Every conditional mutual information term is nonnegative. Hence the
second sum is nonnegative, which gives
\begin{align}
\sum_{(i,j)\in\mathcal{P}(\mathcal{K}')}
I\!\left(Z;D\,\middle|\,Y,D\in\{i,j\}\right)
\ge
\sum_{(i,j)\in\mathcal{P}(\mathcal{K})}
I\!\left(Z;D\,\middle|\,Y,D\in\{i,j\}\right).
\end{align}
Multiplying by $-\lambda$ with $\lambda>0$ and adding $I(Z;Y)$ yields
\begin{align}
    \mathcal{J}_{\mathcal{K}'}(\phi)
    \le
    \mathcal{J}_{\mathcal{K}}(\phi).
\end{align}
Taking the supremum over $\phi$ on both sides gives
\begin{align}
    V_{\mathcal{K}'}\le V_{\mathcal{K}}.
\end{align}
\end{proof}

\subsection{Proof of Proposition~\ref{prop:pairwise-invariance}}
\label{app:proof_pairwise}

\begin{proof}
Under the restricted event $\{D\in\{i,j\}\}$, the domain variable takes
only two values. By the standard characterization of conditional mutual
information,
\begin{align}
I\!\left(Z;D\,\middle|\,Y,D\in\{i,j\}\right)=0
\end{align}
if and only if
\begin{align}
Z\perp D \vert Y,\;D\in\{i,j\}.
\end{align}
This conditional independence means that, for each label value, the
conditional distribution of $Z$ is the same in the two-domain
subpopulation:
\begin{align}
P(Z\vert Y,D=i,D\in\{i,j\})
=
P(Z\vert Y,D=j,D\in\{i,j\}).
\end{align}
Since $\{D=i\}$ and $\{D=j\}$ are subsets of $\{D\in\{i,j\}\}$, this is
equivalent to
\begin{align}
    P(Z\vert Y,D=i)=P(Z\vert Y,D=j).
\end{align}
The converse follows by reversing the same argument.
\end{proof}

\subsection{Proof of Theorem~\ref{thm:shrinking-set}}
\label{app:proof_shrinking_set}

\begin{proof}
By Proposition~\ref{prop:pairwise-invariance}, for every
$(i,j)\in\mathcal{P}(\mathcal{K})$,
\begin{align}
    P(Z\vert Y,D=i)=P(Z\vert Y,D=j).
\end{align}
Thus all domains in $\mathcal{K}$ share the same conditional law of $Z$
given $Y$. Denote this common law by $Q_Y$. Then for every
$i\in\mathcal{K}$,
\begin{align}
    P(Z\vert Y,D=i)=Q_Y.
\end{align}
Marginalizing over domains gives
\begin{align}
P(Z\vert Y)
&=
\sum_{i\in\mathcal{K}}P(Z\vert Y,D=i)P(D=i\vert Y) \nonumber\\
&=
\sum_{i\in\mathcal{K}}Q_Y P(D=i\vert Y)
=
Q_Y.
\end{align}
Therefore,
\begin{align}
P(Z\vert Y,D=i)=P(Z\vert Y)
\qquad
\forall i\in\mathcal{K},
\end{align}
which is exactly $Z\perp D\vert Y$ over $D\in\mathcal{K}$.
\end{proof}

\subsection{Proof of Theorem~\ref{thm:monotonicity}}
\label{app:proof_monotonicity}

\begin{proof}
Under the exact invariance assumption, every pairwise penalty term
vanishes at $\phi_{\mathcal{K}}^{*}$. Therefore,
\begin{align}
V_{\mathcal{K}}
=
\mathcal{J}_{\mathcal{K}}(\phi_{\mathcal{K}}^{*})
=
I(Z_{\mathcal{K}}^{*};Y).
\end{align}
Similarly,
\begin{align}
V_{\mathcal{K}'}
=
I(Z_{\mathcal{K}'}^{*};Y).
\end{align}
From Proposition~\ref{prop:domain-expansion}, if
$\mathcal{K}\subseteq\mathcal{K}'$, then
$V_{\mathcal{K}'}\le V_{\mathcal{K}}$. Substituting the two identities
above yields
\begin{align}
    I(Z_{\mathcal{K}'}^{*};Y)
    \le
    I(Z_{\mathcal{K}}^{*};Y).
\end{align}
\end{proof}

\subsection{Proof of Lemma~\ref{lemma:subset}}

\begin{proof}
Fix a subset index $m$ with $P(S=m)>0$. By the standard characterization of conditional mutual information,
\begin{align}
I(Z^{(m)};D\vert Y,S=m)=0
\end{align}
if and only if $Z^{(m)} \perp D \vert Y,S=m$. This conditional independence holds if and only if, for every label value $y$ and every domain $i$ with
$P(D=i\vert Y=y,S=m)>0$,
\begin{align}
P(Z^{(m)}\vert  Y=y,D=i,S=m)
=
P(Z^{(m)}\vert Y=y,S=m).
\end{align}
Therefore, for any two domains $i,j$ with positive probability under subset $m$,
\begin{align}
P(Z^{(m)}\vert Y=y,D=i,S=m)
=
P(Z^{(m)}\vert Y=y,S=m)
=
P(Z^{(m)}\vert Y=y,D=j,S=m).
\end{align}
This proves the forward direction.
Conversely, suppose that for all domains $i,j$ with positive probability under $S=m$,
\begin{align}
P(Z^{(m)}\vert Y,D=i,S=m)
=
P(Z^{(m)}\vert Y,D=j,S=m).
\end{align}
Then the conditional distribution of $Z^{(m)}$ given $Y$ and $S=m$ is the same for all such domains. Denote this common distribution by $Q_y$. For any domain $i$ with positive probability,
\begin{align}
P(Z^{(m)}\vert Y=y,D=i,S=m)=Q_y.
\end{align}
Moreover,
\begin{align}
P(Z^{(m)}\vert Y=y,S=m)
=
\sum_{k} P(Z^{(m)}\vert Y=y,D=k,S=m)P(D=k\vert Y=y,S=m)
=
Q_y.
\end{align}
Thus,
\begin{align}
P(Z^{(m)}\vert Y=y,D=i,S=m)
=
P(Z^{(m)}\vert Y=y,S=m),
\end{align}
which implies
\begin{align}
Z^{(m)} \perp D \vert Y,S=m.
\end{align}
Hence,
\begin{align}
I(Z^{(m)};D\vert Y,S=m)=0.
\end{align}
\end{proof}

\subsection{Proof of Corollary}

\begin{proof}
By Lemma~\ref{lemma:subset}, subset-conditioned invariance for subset $m$,
\begin{align}
I(Z^{(m)};D\vert Y,S=m)=0,
\end{align}
is equivalent to equality of the class-conditional feature distributions
\begin{align}
P(Z^{(m)}\vert Y,D=i,S=m)
=
P(Z^{(m)}\vert Y,D=j,S=m)
\end{align}
for all domain pairs $i,j$ that have positive probability in subset $m$.

Therefore, enforcing equality of these distributions for every such pair is sufficient to obtain subset-conditioned invariance. Conversely, subset-conditioned invariance implies that all such pairwise class-conditional distributions are equal. Hence subset-conditioned invariance can be enforced through pairwise class-conditional alignment within each subset.
\end{proof}

\subsection{Proof on Theorem~\ref{thm:subset-invariance}}
\begin{proof}
We give a constructive example. Consider four domains partitioned into two subsets,
\begin{align}
\mathcal D_1=\{1,2\}, \qquad \mathcal D_2=\{3,4\}.
\end{align}
Let the label be binary, $Y\in\{-1,+1\}$, with $P(Y=1)=P(Y=-1)=1/2$.

For domains in $\mathcal D_1$, suppose the label-relevant feature is
\begin{align}
X_1=Y,
\end{align}
while another feature $X_2$ is independent noise. For domains in $\mathcal D_2$, suppose the label-relevant feature is
\begin{align}
X_2=Y,
\end{align}
while $X_1$ is independent noise. Thus, the predictive mechanism is shared within each subset but differs across subsets.

Now impose global conditional invariance on a single representation $Z=g(X)$:
\begin{align}
I(Z;D\vert Y)=0.
\end{align}
This requires the conditional distribution $P(Z\vert Y,D)$ to be the same across all four domains. In particular, for fixed $Y=y$, the distribution of $Z$ in domains where $X_1$ is predictive must match the distribution of $Z$ in domains where $X_1$ is noise. Therefore, any component of $Z$ that preserves the subset-specific predictive role of $X_1$ in $\mathcal D_1$ would make $P(Z\vert Y,D)$ differ between $\mathcal D_1$ and $\mathcal D_2$, violating global invariance. The same argument applies to $X_2$: preserving its subset-specific predictive role in $\mathcal D_2$ would distinguish domains in $\mathcal D_2$ from domains in $\mathcal D_1$.

Hence, a globally invariant representation cannot simultaneously preserve both subset-specific predictive factors while satisfying $I(Z;D\vert Y)=0$. It must suppress at least one subset-specific predictive factor to make the conditional representation distribution identical across all domains.

By contrast, define a latent subset variable
\begin{align}
S=1 \quad \text{for } D\in\mathcal D_1,
\qquad
S=2 \quad \text{for } D\in\mathcal D_2.
\end{align}
Let
\begin{align}
Z^{(1)}=X_1,
\qquad
Z^{(2)}=X_2.
\end{align}
Within subset $\mathcal D_1$, $X_1=Y$ in both domains, so
\begin{align}
P(Z^{(1)}\vert Y,D=1,S=1)
=
P(Z^{(1)}\vert Y,D=2,S=1).
\end{align}
Therefore,
\begin{align}
I(Z^{(1)};D\vert Y,S=1)=0.
\end{align}
Similarly, within subset $\mathcal D_2$, $X_2=Y$ in both domains, so
\begin{align}
I(Z^{(2)};D\vert Y,S=2)=0.
\end{align}
Thus, the subset-conditioned representations satisfy subset-conditioned invariance while preserving the predictive factor specific to each subset. This establishes that subset-conditioned invariance can preserve subset-specific predictive structure that global invariance may suppress.
\end{proof}

\section{Details of the MESSI MoE Architecture}
\label{app:messi_architecture}
\subsection{MoE Architecture}
\label{sec:messi-architecture}
\label{app:moe_architecture}

MESSI implements subset-shared invariance with a sparse MoE representation learner. Let \(u=b_\phi(x)\) be the backbone feature. A router produces \(\pi(x)\in\Delta^M\), and each expert \(h_m\) produces \(z^{(m)}=h_m(u)\). The prediction uses the routed mixture
\begin{align}
z(x)=\sum_{m=1}^{M}\pi_m(x)z^{(m)},
\qquad
\hat y=f(z(x)).
\end{align}
The subset-conditioned invariance loss is applied to the expert features \(\{z^{(m)}\}_{m=1}^{M}\). For each expert, class, and domain pair, the aggregated routing mass defines the alignment weight. Thus, routing determines which class-conditional domain pairs each expert aligns, while the classifier predicts from their routed composition.

We follow the sparse MoE design of GMoE~\citep{Li2022SparseMA}. In selected ViT/DeiT blocks, the feed-forward network is replaced by a sparse MoE layer
\begin{align}
f_{\mathrm{MoE}}(x)=\sum_{m=1}^{M}G_m(x)E_m(x),
\end{align}
where \(E_m\) is the \(m\)-th FFN expert and \(G(x)\) is a sparse top-\(k\) routing distribution. We use the same cosine router and sparse inference pattern as GMoE. Unless otherwise stated, all MoE-based variants use the same backbone, expert count, routing configuration, and MoE placement, so the ablations isolate the training objective rather than model capacity.

\subsection{Scalability of Subset-Conditioned OT}
\label{sec:ssi_ot_scalability}

The subset-conditioned OT loss is computed on the current source minibatch. It does not use target-domain samples, memory banks, queues, cached features, cross-batch accumulation, or moving-average feature estimates. At each step, we concatenate the source-domain minibatches and compute labels \(y\), source-domain indices \(d\), routing probabilities \(\pi_m(x)\), and per-expert features \(h_m(x)\). For expert \(m\), class \(c\), and source domain \(i\), define
\begin{align}
\mathcal Z^{(m)}_{i,c}=\{h_m(x):y=c,d=i\},
\qquad
n_{i,c}=|\mathcal Z^{(m)}_{i,c}|.
\end{align}
Only active minibatch cells are used. A class \(c\) contributes only if it appears in at least two source-domain samples in the concatenated minibatch. A domain pair \((i,j)\) contributes only if \(n_{i,c}>0\) and \(n_{j,c}>0\). Missing class-domain cells are skipped. Singleton class-domain groups are allowed, since empirical OT is well-defined for non-empty empirical distributions.

For each active tuple \((m,i,j,c)\), the OT term is weighted by the minibatch routing responsibility
\begin{align}
a^{(m)}_{ijc}
=
\sigma\!\left(\alpha\,\bar{\pi}^{(m)}_{i,c}\right)
\sigma\!\left(\alpha\,\bar{\pi}^{(m)}_{j,c}\right),
\qquad
\bar{\pi}^{(m)}_{i,c}
=
\frac{1}{n_{i,c}}
\sum_{x:y=c,d=i}\pi_m(x).
\end{align}
The routing averages are computed from the current minibatch and detached before weighting the OT discrepancy. Thus, routing selects which class-domain pairs are aligned, while the router itself is trained through the classification objective and routing regularizers.

For each active pair, we compute entropic OT with a differentiable log-space Sinkhorn solver, uniform empirical marginals, and a squared Euclidean cost matrix over expert features. The entropic regularization coefficient and the number of Sinkhorn iterations \(T\) are fixed. Let \(\mathcal C_B\) be the set of active classes in the minibatch, and let
\begin{align}
\mathcal P_B(c)=\{(i,j):i<j,\ n_{i,c}>0,\ n_{j,c}>0\}
\end{align}
be the active source-domain pairs for class \(c\). The minibatch-level complexity is
\begin{align}
O\!\left(
M T
\sum_{c\in\mathcal C_B}
\sum_{(i,j)\in\mathcal P_B(c)}
n_{i,c}n_{j,c}
\right),
\end{align}
plus the same pairwise cost-matrix construction without the factor \(T\). The cost therefore scales with the number and size of active class-domain pairs in the minibatch, not with the total number of dataset classes. This distinction matters for large-class datasets such as DomainNet, where most classes are absent from any single minibatch. The worst case remains expensive when many samples from the same class appear across multiple source domains, which explains the higher training-time cost of the OT variant in Table~\ref{tab:training_time}.

The OT loss is disabled at test time. Inference uses only the encoder, router, expert heads, and classifier, and therefore requires no Sinkhorn iterations or pairwise class-domain computations.

\paragraph{MMD variant for computational comparison.}
We also implement an MMD variant that keeps the same backbone, router, experts, active-pair selection rule, and auxiliary losses, replacing only the OT discrepancy. For each active domain pair \((i,j)\in\mathcal P_B(c)\), the loss is
\begin{align}
\mathcal L_{\mathtt{ssi}}^{\mathrm{MMD}}
=
\sum_{m=1}^{M}
\sum_{c\in\mathcal C_B}
\sum_{(i,j)\in\mathcal P_B(c)}
a^{(m)}_{ijc}
\,\mathrm{MMD}^{2}
\left(
\mathcal Z^{(m)}_{i,c},
\mathcal Z^{(m)}_{j,c}
\right).
\end{align}
With a Gaussian RBF kernel \(k\),
\begin{align}
\mathrm{MMD}^{2}(P,Q)
=
\mathbb E_{z,z'\sim P}k(z,z')
+
\mathbb E_{\tilde z,\tilde z'\sim Q}k(\tilde z,\tilde z')
-
2\mathbb E_{z\sim P,\tilde z\sim Q}k(z,\tilde z).
\end{align}
For groups of size \(n_{i,c}\) and \(n_{j,c}\), MMD requires
\(O(n_{i,c}^{2}+n_{j,c}^{2}+n_{i,c}n_{j,c})\) kernel evaluations and no iterative solver. Entropic OT requires the pairwise cost matrix and \(T\) Sinkhorn iterations, with cost \(O(Tn_{i,c}n_{j,c})\). MMD therefore has a lower training-time constant factor, whereas OT provides geometry-aware matching at higher cost. Both variants have the same inference-time architecture because the discrepancy loss is used only during training.

\subsection{From Latent Subsets to Routing-Weighted Alignment}
\label{app:subset-alignment}
The formulation above uses an ideal latent subset variable $S$, whereas in practice subset membership is not observed. We approximate this latent variable with the learned router by interpreting
\begin{align}
\pi_m(x) \approx P(S=m\vert x).
\end{align}
Under this interpretation, the subset-conditioned invariance constraint
\begin{align}
I(Z^{(m)};D\vert Y,S=m)=0
\end{align}
is equivalent to matching the class-conditional distributions
\begin{align}
P(Z^{(m)}\vert Y=c,D=i,S=m)
=
P(Z^{(m)}\vert Y=c,D=j,S=m)
\end{align}
for domains assigned to subset $m$.
Because $S$ is latent, we cannot directly select the domain pairs belonging to each subset. Instead, we estimate the responsibility of expert $m$ for class $c$ in domain $k$ by the average routing mass
\begin{align}
\rho^{(m)}_{k,c}
=
\frac{1}{|\mathcal D_{k,c}|}
\sum_{(x,y)\in \mathcal D_k,\, y=c}
\pi_m(x).
\end{align}
The product-based weight
\begin{align}
a^{(m)}_{ijc}
=
\sigma(\alpha\rho^{(m)}_{i,c})
\sigma(\alpha\rho^{(m)}_{j,c})
\end{align}
then acts as a soft relaxation of the indicator that domains $i$ and $j$ both belong to subset $m$ for class $c$.
Thus, the empirical OT objective
\begin{align}
\sum_{m,c,i<j}
a^{(m)}_{ijc}
W_\varepsilon
\left(
\mathcal Z^{(m)}_{i,c},
\mathcal Z^{(m)}_{j,c}
\right)
\end{align}
can be viewed as a tractable relaxation of enforcing subset-conditioned invariance. Exact mutual-information constraints are replaced by empirical class-conditional distribution matching, and hard latent subset assignments are replaced by learned soft routing responsibilities.

\clearpage
\section{Experimental Protocols and Datasets}
\label{app:experimental_protocols}

This appendix specifies the datasets and evaluation protocols used in the main text. We first describe the DomainBed leave-one-domain-out protocol for standard OOD generalization. We then introduce Rotated-Colored MNIST, a controlled benchmark that disentangles shape-based prediction from rotation and color-label shortcut shifts, and define the fixed-budget domain-growth protocol used to vary source-domain heterogeneity without increasing the training budget. Finally, we describe the WILDS-iWildCam \citep{koh2021wilds,beery2021iwildcam} protocol for the additional many-domain camera-trap evaluation.

\subsection{DomainBed Protocol}
\label{app:domainbed_protocol}

We evaluate OOD generalization on DomainBed using the standard leave-one-domain-out protocol. For each dataset, one domain is selected as the target domain and all remaining domains are used for training. We follow the training-domain validation criterion: 20\% of each source domain is reserved for validation, and the checkpoint with the highest average source-validation accuracy is selected. Target-domain data are not used for training, validation, hyperparameter tuning, or model selection.

All results are averaged over three random seeds. For each target split, all methods use the same source domains, target domain, and evaluation protocol. The dataset-level score is the mean accuracy over all target domains.

\subsection{Rotated-Colored MNIST Construction}
\label{app:rotated_colored_mnist}

We construct Rotated-Colored MNIST as a controlled diagnostic for domain generalization. Each environment is defined by two nuisance factors: a rotation angle applied to the digit image and a color-label correlation applied after label noise. The prediction task is binary digit classification, so digit shape is the label-relevant factor, while rotation and color correlation define the environment shift.

\begin{wrapfigure}[30]{r}{0.49\textwidth}
\captionsetup{type=algorithm}
\caption{Rotated-Colored MNIST construction.}
\label{alg:rotated_colored_mnist}
\hrule
\begin{lstlisting}[style=meanflowlike]
# E: number of candidate environments
# Dth: rotation step
# Dp: color-correlation step

X, y10 = load_mnist()
X, y10 = shuffle(X, y10)

C = round(1.0 / Dp)
envs = []

for e in range(E):
    th = Dth * floor(e / C)
    p  = Dp  * (e % C)

    Xe, ye10 = split_env(X, y10, e)

    y0 = float(ye10 < 5)
    y  = xor(y0, Bernoulli(0.25))

    Xr = rotate(Xe, th)

    c = xor(y, Bernoulli(p))

    Xc = two_channel(Xr)
    Xc[arange(len(Xc)), 1 - c, :, :] = 0

    envs.append((Xc / 255.0, y))

return envs
\end{lstlisting}
\hrule
\end{wrapfigure}

Algorithm~\ref{alg:rotated_colored_mnist} summarizes the construction. We pool the MNIST training and test sets into \(N=70{,}000\) examples. The ten-way label is converted into a binary label by mapping digits \(\{0,1,2,3,4\}\) to class \(1\) and digits \(\{5,6,7,8,9\}\) to class \(0\). Let \(\bar y\) denote this clean binary label. Following Colored MNIST, we corrupt it with label noise \(\eta_y\sim\mathrm{Bernoulli}(0.25)\):
\begin{align}
y = \bar y \oplus \eta_y,
\label{eq:rotated_colored_label_noise}
\end{align}
where \(\oplus\) denotes XOR. This prevents the binary task from being perfectly deterministic.

\textbf{Environment parameterization.}
We construct candidate environments indexed by \(e\in\{0,\ldots,E-1\}\). Each environment has a rotation angle \(\theta_e\) and a color-flip probability \(p_e\). The grid is controlled by a rotation step \(\Delta_\theta\) and a color step \(\Delta_p\). Let \(C=\mathrm{round}(1/\Delta_p)\) be the number of color levels per rotation. Under the row-major layout,
\begin{align}
\theta_e
=
\Delta_\theta
\left\lfloor \frac{e}{C} \right\rfloor,
\qquad
p_e
=
\Delta_p (e \bmod C).
\label{eq:rotated_colored_env_grid}
\end{align}
Environments in the same row share the same rotation and differ only in color-label correlation, whereas moving across rows adds geometric variation.

In the fixed-budget domain-growth experiments, we use \(\Delta_\theta=45^\circ\) and \(\Delta_p=0.1\). Increasing the number of source environments therefore increases heterogeneity in both rotation and shortcut strength, without increasing the total training budget.

\textbf{Rotation and colorization.}
For each environment \(e\), images assigned to that environment are first rotated by \(\theta_e\) using bilinear interpolation. The rotated grayscale image is then duplicated into two channels. The active color channel is determined by the noisy label \(y\) and the environment-specific flip probability \(p_e\):
\begin{align}
c = y \oplus \eta_c,
\qquad
\eta_c\sim\mathrm{Bernoulli}(p_e).
\label{eq:rotated_colored_color_assignment}
\end{align}
The channel indexed by \(c\) is retained and the other channel is set to zero. The resulting input has shape \(2\times28\times28\) and is normalized to \([0,1]\).

The parameter \(p_e\) controls the color shortcut. When \(p_e=0\), color is perfectly correlated with the noisy label. When \(p_e=0.5\), color is independent of the label. When \(p_e\) approaches \(1\), the correlation is reversed.

\textbf{Purpose of the benchmark.}
This benchmark is used only as a diagnostic. Because the rotation angle and color-label correlation are known by construction, it allows us to test whether a method preserves shape-based predictive structure as source-domain heterogeneity increases, or instead collapses toward an overly restrictive globally aligned representation.

\subsection{Fixed-Budget Domain-Growth Protocol}
\label{app:domain_growth_protocol}

We use Rotated-Colored MNIST to study how DG methods behave as the number of source domains increases. The construction above defines the candidate environment pool; this protocol defines how source domains are sampled under a fixed training budget.

\textbf{Held-out target environment.}
For all domain-growth experiments, we fix the target environment to
\begin{align}
\theta_{\mathrm{test}}=0^\circ,
\qquad
p_{\mathrm{test}}=0.5.
\end{align}
Under \eqref{eq:rotated_colored_env_grid}, this corresponds to \(e_{\mathrm{test}}=5\) when \(\Delta_p=0.1\). Since \(p_{\mathrm{test}}=0.5\), color is independent of the binary label. The target domain therefore tests whether a method learns shape-based prediction rather than relying on the color shortcut.

\textbf{Fixed-budget source-domain expansion.}
We vary the number of source domains as
\begin{align}
K_s\in\{3,5,7,9,11,13,15,17\}.
\end{align}
For each \(K_s\), source domains are sampled from the candidate pool excluding \(e_{\mathrm{test}}\). The total number of training examples is fixed to \(B\), so each selected source domain contributes at most \(\lfloor B/K_s \rfloor\) examples. This separates the effect of increasing source-domain heterogeneity from simply increasing the amount of training data.

\textbf{Sampling, seeds, and model selection.}
For each \(K_s\), we sample \(R\) independent source-domain subsets. Each method is trained with three random seeds per subset, and we report mean and standard deviation over both subset sampling and training seeds. Model selection uses source-validation data only; the target domain is used only for final evaluation.

\begin{table}[h]
\caption{Fixed-budget domain-growth protocol on Rotated-Colored MNIST.}
\label{tab:domain_growth_protocol}
\centering
\small
\setlength{\tabcolsep}{7pt}
\begin{tabular}{lc}
\toprule
Setting & Value \\
\midrule
Source domains & \(K_s\in\{3,5,7,9,11,13,15,17\}\) \\
Total training examples & Fixed to \(B\) for all \(K_s\) \\
Max examples per source domain & \(\lfloor B/K_s \rfloor\) \\
Target environment & \(\theta_{\mathrm{test}}=0^\circ,\; p_{\mathrm{test}}=0.5\) \\
Subset samples & \(R\) per \(K_s\) \\
Training seeds & 3 per subset \\
Model selection & Source-validation only \\
Target usage & Final evaluation only \\
\bottomrule
\end{tabular}
\end{table}

\subsection{WILDS-iWildCam Protocol}
\label{app:iwildcam_protocol}

\textbf{Benchmark.}
We additionally evaluate MESSI on WILDS-iWildCam~\citep{koh2021wilds,beery2021iwildcam}, a real-world many-domain benchmark for camera-trap species recognition. Each example is an image captured by a static camera trap, the label is one of 182 species, and the domain is the camera-trap location. The training split contains 243 source locations with substantial variation in background, viewpoint, illumination, vegetation, and species frequency. This makes iWildCam a natural stress test for global invariance: predictive structure need not be shared uniformly across all camera locations.

\textbf{Splits and model selection.}
We follow the official WILDS five-split protocol, summarized in Table~\ref{tab:iwildcam_splits}. Models are trained on the \textit{train} split. The \textit{ID validation} and \textit{ID test} splits contain held-out images from source camera locations, whereas the \textit{OOD validation} and \textit{OOD test} splits contain images from camera traps disjoint from the training cameras. We select checkpoints using OOD validation and report OOD test macro F1 as the primary generalization metric.

\begin{table}[h]
\centering
\caption{
Split structure of WILDS-iWildCam. The domain is the camera-trap location.
}
\label{tab:iwildcam_splits}
\begin{tabular}{lll}
\toprule
Split & Camera locations & Role \\
\midrule
Train & Source camera traps & Optimization \\
ID Val & Source camera traps & In-distribution validation \\
ID Test & Source camera traps & In-distribution test \\
OOD Val & Unseen camera traps & OOD checkpoint selection \\
OOD Test & Unseen camera traps & Final OOD evaluation \\
\bottomrule
\end{tabular}
\end{table}

\textbf{Compared methods.}
We compare three MoE-based methods under the same DeiT-S/16 backbone, preprocessing pipeline, and six-expert configuration. GMoE~\citep{Li2022SparseMA} is the original sparse MoE baseline. OMoE~\citep{feng2025omoe} augments the MoE baseline with Gram--Schmidt orthogonalization and the importance-CV\({}^2\) load-balancing objective. MESSI uses the same backbone family and is trained with the proposed subset-conditioned invariance objective.

\textbf{Preprocessing.}
Images are resized to \(224\times224\) and normalized with ImageNet statistics. During training, we apply random resized cropping, horizontal flipping, color jittering, and random grayscale augmentation. This setup is lighter than the official WILDS-iWildCam high-resolution pipeline. We therefore use this experiment as a controlled comparison among related MoE variants, not as a direct reproduction of leaderboard-scale WILDS results.

\textbf{Many-domain batching.}
We treat each camera-trap location as a domain. Since the training split contains 243 source locations, using all domains in every minibatch is not tractable. We use stochastic domain-subset batching following the WILDS \texttt{n\_groups\_per\_batch} convention~\citep{koh2021wilds,beery2021iwildcam}. At each step, we sample \(K=4\) source locations uniformly without replacement and draw \(8\) images from each location, forming a balanced minibatch of \(32\) images. For MESSI, these \(K=4\) active domains define \(K(K-1)/2=6\) candidate domain pairs for \(\mathcal L_{\mathtt{ssi}}\).

\textbf{Optimization and metric.}
All methods are trained with Adam using learning rate \(3\times10^{-5}\) and zero weight decay. For MESSI, we set
\(\lambda_{\mathtt{ssi}}=0.01\),
\(\lambda_{\mathrm{sp}}=\lambda_{\mathrm{bal}}=0.02\),
\(\lambda_{\mathrm{div}}=0.02\), and routing-pair temperature \(\alpha=4.0\). For conditional MMD, we use a multi-scale RBF kernel with bandwidths \(\{1,2,4,8,16\}\). We report per-class macro F1, the standard WILDS metric for iWildCam. Macro F1 gives equal weight to each species and is less dominated by frequent classes such as empty frames.

\textbf{Training budget.}
All compared methods are trained for 150K steps. The longer schedule is used because the subset-conditioned invariance objective introduces an additional multi-domain alignment term that slows early optimization before stabilizing. We therefore interpret this experiment as a convergence-aware comparison under a shared backbone and preprocessing pipeline, rather than an equal-step compute benchmark.

\clearpage
\section{Extended Results}
\label{app:extended_results}

\subsection{Extended Results on DomainBed}
\label{app:domainBed_full}
\begin{table}[h]
\caption{%
OOD accuracy (\%) under the DomainBed training-domain validation criterion, averaged over three seeds. \textbf{MESSI-Ti} and \textbf{MESSI-S} use DeiT-Ti/16 and DeiT-S/16 backbones, respectively. Best results are in \textbf{bold}; second-best results are \underline{underlined}.
}
\label{tab:main_extend}
\centering
\small
\setlength{\tabcolsep}{4.5pt}
\renewcommand{\arraystretch}{1.05}
\begin{tabular*}{\textwidth}{@{\extracolsep{\fill}}llccccc}
\toprule
Method & Backbone & PACS $\uparrow$ & OfficeHome $\uparrow$ & TerraInc $\uparrow$ & DomainNet $\uparrow$ & Avg. $\uparrow$ \\
\midrule
ERM~\citep{Vapnik1991PrinciplesOR}
  & ResNet-50
  & 83.8\ms{0.8} & 66.6\ms{0.4} & 47.2\ms{0.8} & 41.6\ms{0.2} & 59.8 \\
IRM~\citep{arjovsky2019invariant}
  & ResNet-50
  & 81.1\ms{0.5} & 58.2\ms{0.2} & 38.7\ms{0.8} & 30.6\ms{1.0} & 52.2 \\
MMD~\citep{Li2018DomainGW}
  & ResNet-50
  & 81.4\ms{0.7} & 60.1\ms{0.5} & 42.2\ms{0.9} & 20.5\ms{0.3} & 51.1 \\
DANN~\citep{ganin2016domain}
  & ResNet-50
  & 79.4\ms{0.3} & 59.6\ms{0.4} & 38.1\ms{0.8} & 31.6\ms{0.1} & 52.2 \\
MixStyle~\citep{Zhou2021DomainGW}
  & ResNet-50
  & 82.7\ms{0.4} & 59.8\ms{0.7} & 41.1\ms{1.0} & 34.1\ms{0.1} & 54.4 \\
CORAL~\citep{Sun2016DeepCC}
  & ResNet-50
  & 83.5\ms{0.4} & 66.0\ms{0.5} & 43.8\ms{0.9} & 39.2\ms{0.3} & 58.1 \\
Fishr~\citep{Ram2021FishrIG}
  & ResNet-50
  & 85.5\ms{0.2} & 68.6\ms{0.2} & 47.4\ms{1.6} & 41.7\ms{0.3} & 60.8 \\
SAGM~\citep{Wang2023SharpnessAwareGM}
  & ResNet-50
  & 86.4\ms{1.2} & 69.4\ms{0.2} & 48.8\ms{1.0} & 43.2\ms{0.4} & 61.9 \\
LFME~\citep{Chen2024LFMEAS}
  & ResNet-50
  & 84.9\ms{0.4} & 68.5\ms{0.2} & \textbf{49.5}\ms{\textbf{0.8}} & 38.6\ms{0.2} & 60.4 \\
\midrule
ERM~\citep{Vapnik1991PrinciplesOR}
  & DeiT-S/16
  & 86.2\ms{0.1} & 72.2\ms{0.4} & 42.0\ms{0.8} & 47.3\ms{0.2} & 61.9 \\
DynMoE~\citep{Guo2024DynamicMO}
  & DeiT-S/16
  & 85.2\ms{0.4} & 73.4\ms{0.3} & 44.5\ms{0.6} & 45.9\ms{0.4} & 62.3 \\
GMoE~\citep{Li2022SparseMA}
  & DeiT-S/16
  & 87.3\ms{0.1} & 73.5\ms{0.1} & 48.2\ms{0.5} & \underline{47.8\ms{0.5}} & 64.2 \\
OMoE~\citep{feng2025omoe}
  & DeiT-S/16
  & 87.1\ms{0.3} & 73.2\ms{0.1} & 46.0\ms{0.5} & 46.1\ms{0.2} & 63.1 \\
\midrule
\rowcolor{gray!10}
\textbf{MESSI-Ti-OT (Ours)}
  & DeiT-Ti/16
  & 86.0\ms{0.3} & 69.6\ms{0.2} & 42.5\ms{1.1} & 43.3\ms{0.4} & 60.4 \\

\rowcolor{gray!10}
\textbf{MESSI-S-MMD (Ours)}
  & DeiT-S/16
  & \underline{89.3\ms{0.6}} & \underline{74.0\ms{0.3}} & 47.5\ms{0.4} & 46.2\ms{0.4} & \underline{64.3} \\

\rowcolor{gray!10}
\textbf{MESSI-S-OT (Ours)}
  & DeiT-S/16
  & \textbf{90.9}\ms{\textbf{0.7}} & \textbf{76.1}\ms{\textbf{0.2}} & \underline{49.3\ms{0.2}} & \textbf{48.6}\ms{\textbf{0.3}} & \textbf{66.2} \\
\bottomrule
\end{tabular*}
\end{table}

\subsection{DomainBed Per-Domain Results}
\label{app:domainbed_per_domain}

Tables~\ref{tab:app_pacs}--\ref{tab:app_domainnet} report per-target-domain
accuracy. Best results are shown in \textbf{bold}, and second-best results are
\underline{underlined}. The average column matches Table~\ref{tab:main}.

\begin{table}[h]
\caption{Per-domain accuracy on PACS.}
\label{tab:app_pacs}
\centering
\small
\setlength{\tabcolsep}{6pt}
\begin{tabular}{lccccc}
\toprule
Method & Art $\uparrow$ & Cartoon $\uparrow$ & Photo $\uparrow$ & Sketch $\uparrow$ & Avg. $\uparrow$ \\
\midrule
ERM
  & 86.2\ms{0.8} & 76.0\ms{0.7} & 95.9\ms{0.4} & 77.1\ms{0.9} & 83.8\ms{0.8} \\
IRM
  & 83.0\ms{0.6} & 74.5\ms{0.4} & 95.8\ms{0.3} & 71.1\ms{0.7} & 81.1\ms{0.5} \\
MMD
  & 84.0\ms{0.8} & 74.8\ms{0.6} & 95.7\ms{0.3} & 71.1\ms{1.0} & 81.4\ms{0.7} \\
CORAL
  & 85.2\ms{0.7} & 76.6\ms{1.0} & 95.1\ms{0.2} & 77.1\ms{0.8} & 83.5\ms{0.4} \\
Fishr
  & 87.4\ms{0.3} & 78.4\ms{0.4} & 97.7\ms{0.1} & 78.5\ms{0.5} & 85.5\ms{0.2} \\
SAGM
  & 87.5\ms{1.6} & 80.7\ms{1.4} & 96.4\ms{0.6} & 81.0\ms{3.2} & 86.4\ms{1.2} \\
DynMoE
  & 89.0\ms{0.3} & 81.2\ms{0.5} & \underline{98.9\ms{0.2}} & 71.7\ms{0.7} & 85.2\ms{0.4} \\
GMoE
  & 89.9\ms{0.4} & 83.7\ms{0.5} & 98.9\ms{0.1} & 76.7\ms{0.6} & 87.3\ms{0.1} \\
OMoE
  & 89.6\ms{0.4} & 83.4\ms{0.5} & 98.8\ms{0.2} & 76.6\ms{0.7} & 87.1\ms{0.3} \\
\rowcolor{gray!10}
\textbf{MESSI-Ti-OT (Ours)}
  & 88.3\ms{0.9} & 78.1\ms{1.1} & 98.1\ms{0.5} & 79.3\ms{2.5} & 86.0\ms{0.3} \\
\rowcolor{gray!10}
\textbf{MESSI-S-MMD (Ours)}
  & \underline{90.4\ms{0.6}} & \underline{86.7\ms{0.8}} & 98.7\ms{0.4} & \underline{81.4\ms{0.6}} & \underline{89.3\ms{0.6}} \\
\rowcolor{gray!10}
\textbf{MESSI-S-OT (Ours)}
  & \textbf{91.6}\ms{\textbf{0.3}} & \textbf{88.7}\ms{\textbf{0.9}} & \textbf{99.1}\ms{\textbf{0.3}} & \textbf{84.1}\ms{\textbf{1.4}} & \textbf{90.9}\ms{\textbf{0.7}} \\
\bottomrule
\end{tabular}
\end{table}
\clearpage

\begin{table}[h]
\caption{Per-domain accuracy on OfficeHome.}
\label{tab:app_officehome}
\centering
\small
\setlength{\tabcolsep}{6pt}
\begin{tabular}{lccccc}
\toprule
Method & Art $\uparrow$ & Clipart $\uparrow$ & Product $\uparrow$ & Real-World $\uparrow$ & Avg. $\uparrow$ \\
\midrule
ERM
  & 61.8\ms{0.8} & 52.5\ms{0.5} & 75.7\ms{0.4} & 76.4\ms{0.4} & 66.6\ms{0.4} \\
IRM
  & 52.1\ms{0.3} & 44.3\ms{0.2} & 65.5\ms{0.3} & 70.9\ms{0.2} & 58.2\ms{0.2} \\
MMD
  & 54.0\ms{0.5} & 46.0\ms{0.4} & 68.2\ms{0.5} & 72.2\ms{0.6} & 60.1\ms{0.5} \\
CORAL
  & 61.8\ms{0.4} & 52.7\ms{0.6} & 74.2\ms{0.5} & 75.3\ms{0.5} & 66.0\ms{0.5} \\
Fishr
  & 64.1\ms{0.2} & 54.9\ms{0.3} & 76.8\ms{0.3} & 78.6\ms{0.2} & 68.6\ms{0.2} \\
SAGM
  & 64.8\ms{0.7} & 55.3\ms{0.6} & 78.1\ms{0.7} & 79.4\ms{0.6} & 69.4\ms{0.2} \\
DynMoE
  & 71.6\ms{0.3} & 59.5\ms{0.4} & 79.6\ms{0.2} & 82.7\ms{0.3} & 73.4\ms{0.3} \\
GMoE
  & 71.5\ms{0.1} & 58.0\ms{0.3} & \underline{81.2\ms{0.5}} & \underline{83.2\ms{0.1}} & 73.5\ms{0.1} \\
OMoE
  & 71.1\ms{0.2} & 57.8\ms{0.3} & 80.9\ms{0.4} & 83.0\ms{0.2} & 73.2\ms{0.1} \\
\rowcolor{gray!10}
\textbf{MESSI-Ti-OT (Ours)}
  & 66.5\ms{0.5} & 54.2\ms{0.3} & 77.1\ms{0.7} & 80.6\ms{1.2} & 69.6\ms{0.2} \\
\rowcolor{gray!10}
\textbf{MESSI-S-MMD (Ours)}
  & \underline{74.2\ms{0.2}} & \underline{59.9\ms{0.3}} & 80.4\ms{0.4} & 81.5\ms{0.2} & \underline{74.0\ms{0.3}} \\
\rowcolor{gray!10}
\textbf{MESSI-S-OT (Ours)}
  & \textbf{76.1}\ms{\textbf{0.7}} & \textbf{61.2}\ms{\textbf{0.3}} & \textbf{82.2}\ms{\textbf{0.4}} & \textbf{84.9}\ms{\textbf{0.4}} & \textbf{76.1}\ms{\textbf{0.2}} \\
\bottomrule
\end{tabular}
\end{table}

\begin{table}[h]
\caption{Per-domain accuracy on TerraIncognita.}
\label{tab:app_terra}
\centering
\small
\setlength{\tabcolsep}{6pt}
\begin{tabular}{lccccc}
\toprule
Method & L100 $\uparrow$ & L38 $\uparrow$ & L43 $\uparrow$ & L46 $\uparrow$ & Avg. $\uparrow$ \\
\midrule
ERM
  & 50.8\ms{1.8} & 42.5\ms{0.7} & 57.9\ms{0.6} & 37.6\ms{1.2} & 47.2\ms{0.8} \\
IRM
  & 42.0\ms{1.0} & 31.2\ms{0.7} & 46.5\ms{0.6} & 35.1\ms{0.8} & 38.7\ms{0.8} \\
MMD
  & 47.2\ms{1.2} & 35.6\ms{0.8} & 50.6\ms{0.7} & 35.4\ms{1.0} & 42.2\ms{0.9} \\
CORAL
  & 46.1\ms{1.0} & 39.6\ms{3.2} & 53.3\ms{0.7} & 36.2\ms{0.8} & 43.8\ms{0.9} \\
Fishr
  & 52.1\ms{1.6} & \underline{42.5\ms{2.1}} & 56.7\ms{0.8} & 38.3\ms{1.5} & 47.4\ms{1.6} \\
SAGM
  & 52.7\ms{3.2} & \textbf{43.5}\ms{\textbf{4.0}} & \textbf{58.7}\ms{\textbf{0.7}} & 40.3\ms{1.4} & \underline{48.8\ms{1.0}} \\
DynMoE
  & 54.8\ms{0.8} & 36.4\ms{0.7} & 49.0\ms{0.5} & 37.8\ms{0.9} & 44.5\ms{0.6} \\
GMoE
  & 56.6\ms{0.7} & 36.2\ms{0.9} & \underline{58.0\ms{0.6}} & \underline{42.0\ms{0.8}} & 48.2\ms{0.5} \\
OMoE
  & 55.5\ms{0.8} & 34.2\ms{1.4} & 55.3\ms{0.6} & 39.0\ms{0.7} & 46.0\ms{0.5} \\
\rowcolor{gray!10}
\textbf{MESSI-Ti-OT (Ours)}
  & 53.6\ms{1.9} & 23.1\ms{1.2} & 52.2\ms{0.8} & 41.2\ms{0.3} & 42.5\ms{1.1} \\
\rowcolor{gray!10}
\textbf{MESSI-S-MMD (Ours)}
  & \underline{57.1\ms{0.5}} & 36.9\ms{0.5} & 54.8\ms{0.4} & 41.2\ms{0.3} & 47.5\ms{0.4} \\
\rowcolor{gray!10}
\textbf{MESSI-S-OT (Ours)}
  & \textbf{60.2}\ms{\textbf{0.7}} & 38.3\ms{0.5} & 56.7\ms{0.3} & \textbf{42.1}\ms{\textbf{0.2}} & \textbf{49.3}\ms{\textbf{0.2}} \\
\bottomrule
\end{tabular}
\end{table}

\begin{table}[h]
\caption{Per-domain accuracy on DomainNet.}
\label{tab:app_domainnet}
\centering
\small
\setlength{\tabcolsep}{4pt}
\resizebox{\textwidth}{!}{
\begin{tabular}{lccccccc}
\toprule
Method & Real $\uparrow$ & Clipart $\uparrow$ & Painting $\uparrow$ & Sketch $\uparrow$ & Infograph $\uparrow$ & Quickdraw $\uparrow$ & Avg. $\uparrow$ \\
\midrule
ERM
  & 61.2\ms{0.4} & 58.8\ms{0.3} & 46.7\ms{0.4} & 50.1\ms{0.7} & 19.4\ms{0.3} & 13.4\ms{0.2} & 41.6\ms{0.2} \\
IRM
  & 48.0\ms{0.8} & 42.5\ms{0.7} & 34.2\ms{0.8} & 36.8\ms{1.0} & 12.1\ms{0.5} & 10.0\ms{0.6} & 30.6\ms{1.0} \\
MMD
  & 33.8\ms{0.5} & 29.3\ms{0.4} & 23.0\ms{0.4} & 24.8\ms{0.5} & 6.7\ms{0.2} & 5.4\ms{0.2} & 20.5\ms{0.3} \\
CORAL
  & 57.6\ms{0.3} & 56.1\ms{0.3} & 44.8\ms{0.4} & 47.6\ms{0.6} & 18.3\ms{0.3} & 10.8\ms{0.3} & 39.2\ms{0.3} \\
Fishr
  & 60.0\ms{0.2} & 58.4\ms{0.1} & 47.1\ms{0.2} & 50.1\ms{0.2} & 20.6\ms{0.1} & 14.0\ms{0.1} & 41.7\ms{0.0} \\
SAGM
  & 61.9\ms{0.2} & 61.1\ms{0.3} & 49.7\ms{0.4} & 52.2\ms{0.4} & 20.5\ms{0.2} & 13.8\ms{0.6} & 43.2\ms{0.4} \\
DynMoE
  & 62.4\ms{0.3} & 60.1\ms{0.4} & 49.9\ms{0.3} & 52.0\ms{0.5} & 22.4\ms{0.2} & \textbf{28.6}\ms{\textbf{0.4}} & 45.9\ms{0.4} \\
GMoE
  & \underline{68.8\ms{0.3}} & \underline{66.1\ms{0.4}} & \underline{55.2\ms{0.4}} & \underline{55.1\ms{0.5}} & \underline{24.7\ms{0.5}} & 16.9\ms{0.4} & \underline{47.8\ms{0.5}} \\
OMoE
  & 66.6\ms{0.1} & 63.8\ms{0.2} & 53.1\ms{0.2} & 53.2\ms{0.4} & 23.5\ms{0.3} & 16.1\ms{0.2} & 46.1\ms{0.2} \\
\rowcolor{gray!10}
\textbf{MESSI-Ti-OT (Ours)}
  & 62.4\ms{0.5} & 60.1\ms{0.3} & 49.1\ms{0.3} & 49.5\ms{0.3} & 22.1\ms{0.4} & 16.4\ms{0.3} & 43.3\ms{0.4} \\
\rowcolor{gray!10}
\textbf{MESSI-S-MMD (Ours)}
  & 67.4\ms{0.3} & 64.9\ms{0.3} & 53.4\ms{0.5} & 53.7\ms{0.3} & 24.1\ms{0.5} & 13.7\ms{0.3} & 46.2\ms{0.4} \\
\rowcolor{gray!10}
\textbf{MESSI-S-OT (Ours)}
  & \textbf{69.9}\ms{\textbf{0.2}} & \textbf{66.9}\ms{\textbf{0.2}} & \textbf{55.7}\ms{\textbf{0.4}} & \textbf{55.9}\ms{\textbf{0.2}} & \textbf{25.6}\ms{\textbf{0.5}} & \underline{17.8\ms{0.3}} & \textbf{48.6}\ms{\textbf{0.3}} \\
\bottomrule
\end{tabular}
}
\end{table}
\clearpage

\subsection{Fixed-Budget Domain-Growth results}
\label{app:rotated_colored_mnist_results}

For the fixed-budget domain-growth experiment, we vary
\(K_s\in\{3,5,7,9,11,13,15,17\}\) while keeping the total source training budget and the held-out target environment fixed. The target has \(p_{\mathrm{test}}=0.5\), so color is independent of the binary label. Target accuracy therefore reflects whether a method learns shape-based prediction rather than relying on the color shortcut.

We summarize robustness using the peak target accuracy and the final accuracy at \(K_{\max}=17\). We define
\(\mathrm{Drop}=\mathrm{Peak}-\mathrm{Acc}(K_{\max})\) and
\(\mathrm{RelDrop}=\mathrm{Drop}/\mathrm{Peak}\). RelDrop prevents flat but low-performing methods from appearing robust simply because they never reach high accuracy.

\begin{table}[h]
\caption{
Fixed-budget domain-growth summary on Rotated-Colored MNIST. Drop is
\(\mathrm{Peak}-\mathrm{Acc}(K_{\max})\) with \(K_{\max}=17\), and RelDrop is
\(\mathrm{Drop}/\mathrm{Peak}\). Best results are shown in \textbf{bold}.
}
\label{tab:domain_growth_summary_app}
\centering
\small
\setlength{\tabcolsep}{6pt}
\begin{tabular}{lcccc}
\toprule
Method & Peak Acc. $\uparrow$ & $\mathrm{Acc}(K_{\max})$ $\uparrow$ & Drop $\downarrow$ & RelDrop $\downarrow$ \\
\midrule
CORAL ($\gamma{=}1$)   & 71.96 & 69.38 & 2.58 & 3.59\% \\
CORAL ($\gamma{=}10$)  & 71.64 & 68.77 & 2.87 & 4.01\% \\
CORAL ($\gamma{=}100$) & 70.10 & 66.46 & 3.64 & 5.19\% \\
CORAL ($\gamma{=}250$) & 68.75 & 63.67 & 5.08 & 7.39\% \\
CORAL ($\gamma{=}500$) & 67.69 & 59.78 & 7.91 & 11.69\% \\
\rowcolor{gray!10}
\textbf{MESSI (Ours)} & \textbf{75.31} & \textbf{73.45} & \textbf{1.86} & \textbf{2.47\%} \\
\bottomrule
\end{tabular}
\end{table}

\begin{table}[h]
\caption{
Fixed-budget source-domain expansion on Rotated-Colored MNIST for
\(K_s\in\{3,5,7,9\}\). The target environment and total source budget are fixed.
Best results are shown in \textbf{bold}.
}
\label{tab:domain_growth_full_table_low_K}
\centering
\small
\setlength{\tabcolsep}{6pt}
\begin{tabular}{lcccc}
\toprule
Method & $K_s=3$ $\uparrow$ & $K_s=5$ $\uparrow$ & $K_s=7$ $\uparrow$ & $K_s=9$ $\uparrow$ \\
\midrule
CORAL ($\gamma{=}1$)   & \textbf{50.84} & 50.79 & 68.45 & 71.96 \\
CORAL ($\gamma{=}10$)  & 50.79 & 50.79 & 68.80 & 71.64 \\
CORAL ($\gamma{=}100$) & 50.79 & 50.82 & 68.85 & 70.10 \\
CORAL ($\gamma{=}250$) & 50.79 & 50.86 & 65.65 & 68.75 \\
CORAL ($\gamma{=}500$) & 50.79 & 50.79 & 59.95 & 64.89 \\
\rowcolor{gray!10}
\textbf{MESSI (Ours)} & 50.79 & \textbf{57.14} & \textbf{74.00} & \textbf{75.31} \\
\bottomrule
\end{tabular}
\end{table}

\begin{table}[h]
\caption{
Continuation of Table~\ref{tab:domain_growth_full_table_low_K} for
\(K_s\in\{11,13,15,17\}\). Best results are shown in \textbf{bold}.
}
\label{tab:domain_growth_full_table_high_K}
\centering
\small
\setlength{\tabcolsep}{6pt}
\begin{tabular}{lcccc}
\toprule
Method & $K_s=11$ $\uparrow$ & $K_s=13$ $\uparrow$ & $K_s=15$ $\uparrow$ & $K_s=17$ $\uparrow$ \\
\midrule
CORAL ($\gamma{=}1$)   & 70.44 & 71.00 & 70.20 & 69.38 \\
CORAL ($\gamma{=}10$)  & 70.91 & 71.00 & 68.92 & 68.77 \\
CORAL ($\gamma{=}100$) & 69.10 & 68.59 & 66.35 & 66.46 \\
CORAL ($\gamma{=}250$) & 67.30 & 67.84 & 63.99 & 63.67 \\
CORAL ($\gamma{=}500$) & 67.69 & 66.17 & 62.81 & 59.78 \\
\rowcolor{gray!10}
\textbf{MESSI (Ours)} & \textbf{74.80} & \textbf{74.10} & \textbf{74.99} & \textbf{73.45} \\
\bottomrule
\end{tabular}
\end{table}

\textbf{Results.}
Table~\ref{tab:domain_growth_summary_app} shows that MESSI achieves the highest peak accuracy, the highest final accuracy at \(K_{\max}=17\), and the smallest relative drop. The full results in Tables~\ref{tab:domain_growth_full_table_low_K} and~\ref{tab:domain_growth_full_table_high_K} show the same pattern across individual source-domain budgets. At small \(K_s\), all methods are close to chance because the source environments provide limited evidence for separating shape from color. Once source diversity becomes informative, MESSI improves sharply and remains stable as \(K_s\) increases, whereas stronger CORAL penalties become increasingly brittle. This supports the role of routing-conditioned subset alignment under increasing domain heterogeneity.

\subsection{Additional Results on WILDS-iWildCam}
\label{app:iwildcam_results}

Table~\ref{tab:iwildcam_results} reports additional results on WILDS-iWildCam \citep{koh2021wilds, beery2021iwildcam} under the protocol in Appendix~\ref{app:iwildcam_protocol}. This experiment is a controlled many-domain comparison among closely related MoE variants, not a leaderboard-oriented WILDS evaluation.

\begin{table}[h]
\centering
\caption{
Additional results on WILDS-iWildCam. All methods use the same DeiT-S/16 backbone, \(224\times224\) preprocessing pipeline, and training budget. We report macro F1 on OOD and ID splits. OOD test is the primary evaluation on unseen camera locations. Best results are shown in \textbf{bold}.
}
\label{tab:iwildcam_results}
\begin{tabular}{lccccc}
\toprule
\multirow{2}{*}{Method}
& \multirow{2}{*}{Steps}
& \multicolumn{2}{c}{OOD F1 $\uparrow$}
& \multicolumn{2}{c}{ID F1 $\uparrow$} \\
\cmidrule(lr){3-4}\cmidrule(lr){5-6}
& & Val & Test & Val & Test \\
\midrule
OMoE
& 150K & 29.14 & 27.96 & 48.57 & 47.26 \\
GMoE
& 150K & 29.06 & 26.68 & 48.84 & 47.28 \\
\rowcolor{gray!10}
\textbf{MESSI (Ours)}
& 150K & \textbf{35.49} & \textbf{33.16} & \textbf{52.86} & \textbf{51.14} \\
\bottomrule
\end{tabular}
\end{table}

\textbf{Results.}
Table~\ref{tab:iwildcam_results} shows that MESSI achieves the best macro F1 on all four evaluation splits. On the primary OOD test split, MESSI improves over OMoE by \(+5.20\) F1 and over GMoE by \(+6.48\) F1. The same pattern appears on the ID test split, where MESSI improves over OMoE by \(+3.88\) F1 and over GMoE by \(+3.86\) F1. These gains are obtained under the same backbone, preprocessing pipeline, number of experts, and training budget, indicating that the improvement is not explained by generic sparse MoE capacity alone.

\textbf{Interpretation.}
The absolute scores are not intended to match heavily tuned WILDS leaderboard results, since we use lower-resolution inputs and a lightweight preprocessing pipeline. The purpose of this experiment is narrower: to test whether MESSI retains a relative advantage over closely related MoE baselines in a real many-domain setting. The consistent improvement on both OOD and ID splits supports the role of routing-conditioned subset alignment when domains are numerous, heterogeneous, and only partially share predictive structure.

\section{Ablations Isolate Routing-Conditioned Subset Alignment}
\label{app:ablations_mechanisms}

\subsection{Ablation on Training Objectives}
\label{app:objective_ablations}

The following ablations test whether MESSI's gains come from routing-induced subset-conditioned invariance rather than MoE capacity, generic feature alignment, or individual regularizers. Unless otherwise stated, all variants use the same backbone, number of experts, optimizer, augmentation pipeline, and model-selection criterion. We report OOD accuracy together with diagnostics tied to the intended role of each term: routing entropy for assignment confidence, expert-load standard deviation for usage imbalance, off-diagonal cosine similarity for expert redundancy, and class-conditional routing JS divergence for routing consistency. Diagnostics are computed on held-out source-validation splits; target-domain data are used only for final accuracy evaluation.

\begin{table}[h]
\caption{
Objective leave-one-out ablation on PACS, averaged over the four
leave-one-domain-out runs. Diagnostic columns are rescaled as indicated in
the headers. Best results are shown in \textbf{bold}.
}
\label{tab:objective_leave_one_out}
\centering
\small
\setlength{\tabcolsep}{3.5pt}
\resizebox{\textwidth}{!}{
\begin{tabular}{lccccc}
\toprule
Variant
& PACS Acc. $\uparrow$
& \shortstack{Routing Ent. $\downarrow$\\$(\times 10^{-2})$}
& \shortstack{Load Std. $\downarrow$\\$(\times 10^{-1})$}
& \shortstack{Offdiag Cos. $\downarrow$\\$(\times 10^{-2})$}
& \shortstack{Routing JS $\downarrow$\\$(\times 10^{-4})$} \\
\midrule
\textbf{MESSI (Ours)}
& \textbf{90.90}\ms{\textbf{0.70}}
& 0.7\ms{0.2}
& 3.6\ms{0.2}
& \textbf{-0.2}\ms{\textbf{0.3}}
& \textbf{0.1}\ms{\textbf{0.0}} \\

w/o $\mathcal L_{\mathtt{ssi}}$
& 86.89\ms{0.75}
& 1.2\ms{0.3}
& 3.5\ms{0.2}
& -0.1\ms{0.3}
& 2.0\ms{0.6} \\

w/o $\mathcal L_{\mathtt{sp}}$
& 87.06\ms{0.90}
& 53.2\ms{5.5}
& 3.0\ms{0.3}
& 0.0\ms{0.4}
& 1.0\ms{0.4} \\

w/o $\mathcal L_{\mathtt{bal}}$
& 88.82\ms{0.65}
& \textbf{0.3}\ms{\textbf{0.1}}
& 3.7\ms{0.3}
& -0.1\ms{0.3}
& \textbf{0.1}\ms{\textbf{0.1}} \\

w/o $\mathcal L_{\mathtt{div}}$
& 88.79\ms{0.70}
& 5.8\ms{1.2}
& 1.2\ms{0.2}
& 35.6\ms{3.0}
& 33.0\ms{4.0} \\

w/o $\mathcal L_{\mathtt{sp}},\mathcal L_{\mathtt{bal}},\mathcal L_{\mathtt{div}}$
& 88.59\ms{0.80}
& 39.7\ms{4.5}
& \textbf{1.0}\ms{\textbf{0.1}}
& 38.9\ms{3.5}
& 30.0\ms{3.5} \\
\bottomrule
\end{tabular}
}
\end{table}

\textbf{Metric definitions.}
Routing entropy is
\(\mathbb{E}_{x}[-\sum_m \pi_m(x)\log \pi_m(x)]\).
Load standard deviation is the standard deviation of average expert usage
\(\mathbb{E}_{x}[\pi_m(x)]\) across experts. Offdiag Cos. is the mean
off-diagonal cosine similarity between expert representations. Routing JS is
the average Jensen--Shannon divergence between domain-specific class routing
prototypes and their class-wise template. Load standard deviation should be read
together with routing entropy: a low load standard deviation can also arise from
diffuse routing that uses all experts weakly.

\textbf{Results.}
Table~\ref{tab:objective_leave_one_out} shows that the full objective achieves the highest PACS accuracy. Removing \(\mathcal L_{\mathtt{ssi}}\) gives the largest drop, from \(90.90\%\) to \(86.89\%\), indicating that routing-conditioned subset alignment is the main accuracy-critical component. Removing \(\mathcal L_{\mathtt{sp}}\) sharply increases routing entropy, so expert assignments become diffuse. Removing \(\mathcal L_{\mathtt{bal}}\) has a smaller effect but slightly increases load imbalance. Removing \(\mathcal L_{\mathtt{div}}\) strongly increases off-diagonal cosine similarity and routing JS, showing that experts become redundant and less class-consistent without explicit specialization. Removing all routing and specialization regularizers produces diffuse routing and high expert redundancy, with lower accuracy than the full model. Overall, \(\mathcal L_{\mathtt{ssi}}\) drives the OOD gain, while the auxiliary terms stabilize the sparse, balanced, and non-redundant expert decomposition.

\subsection{Expert Specialization}
\label{app:expert_specialization}

Figure~\ref{fig:app_expert_analysis} and Table~\ref{tab:expert_specialization}
analyze expert specialization on OfficeHome. The goal of this diagnostic is to
test whether MESSI learns distinct expert representations, rather than gaining
accuracy only from increased MoE capacity.

\begin{figure}[h]
  \centering
  \includegraphics[
    width=\linewidth,
    trim=16 8 8 8,
    clip
  ]{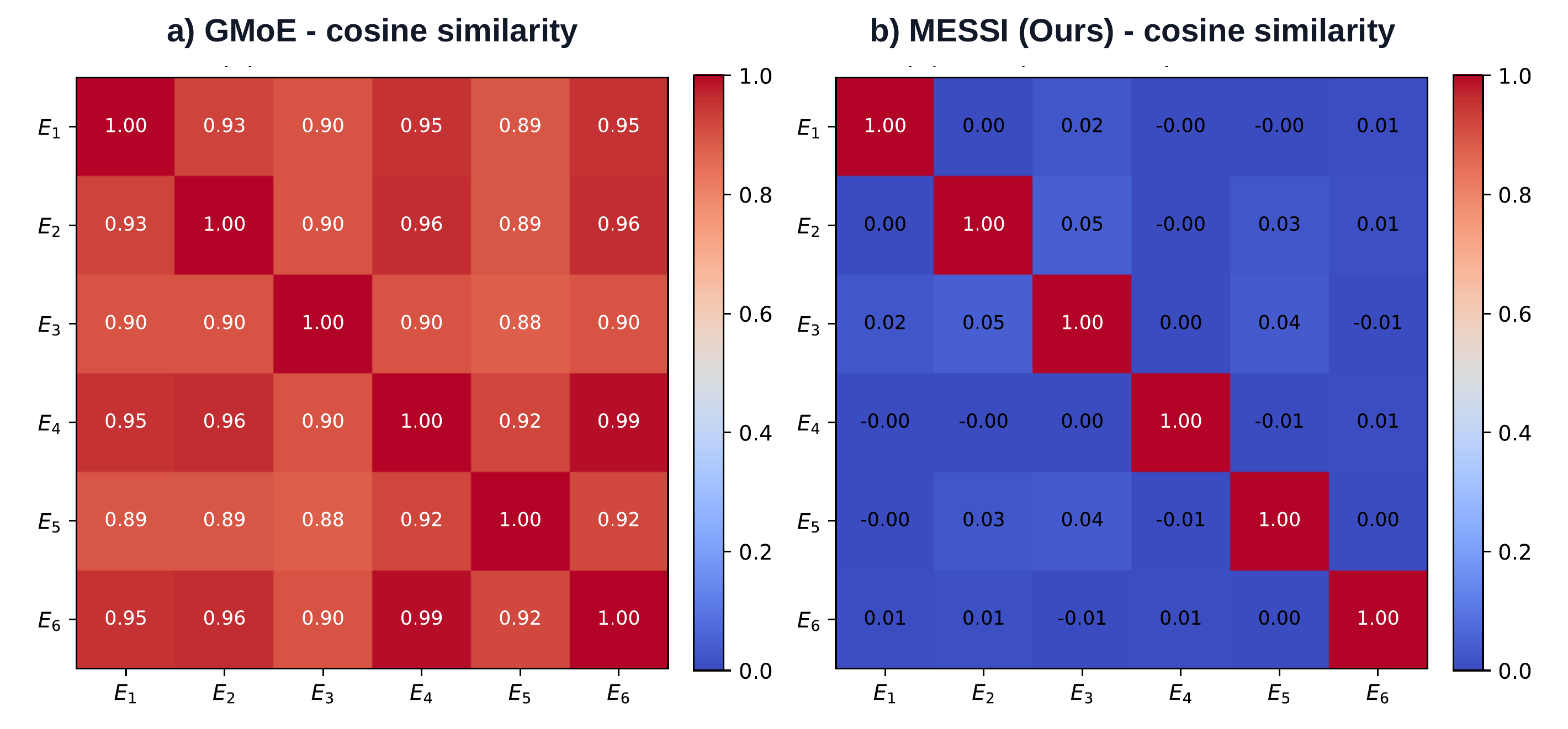}
  \caption{
  Expert specialization on OfficeHome.
  (a) Pairwise cosine similarity between expert representations in GMoE.
  (b) Pairwise cosine similarity between expert representations in MESSI.
  Lower off-diagonal similarity indicates stronger expert specialization.
  }
  \label{fig:app_expert_analysis}
\end{figure}

\textbf{Protocol.}
We follow the DomainBed leave-one-domain-out protocol on OfficeHome, using Art,
Clipart, Product, and Real-World as the held-out domains in turn. All
diagnostics are computed on the held-out target split, so the analysis reflects
expert behavior on unseen domains. We use OfficeHome because its four visually
diverse domains and fine-grained label space make expert redundancy easier to
observe than on smaller benchmarks such as PACS.

\textbf{Metrics.}
Both models use \(M=6\) experts. For each input \(x\), we extract the output of
each expert, denoted by \(h_m(x)\). For MESSI, these are the outputs of the MoE
head. For GMoE, we use the per-expert outputs from the final MoE layer before
routing aggregation. We quantify specialization with the pairwise cosine matrix
\begin{align}
C(x)_{m,n}=\cos\!\left(h_m(x),h_n(x)\right),
\end{align}
averaged over samples and then over held-out domains. Lower off-diagonal values
indicate less redundancy across experts.

Table~\ref{tab:expert_specialization} further reports three summary metrics.
Mean Offdiag Cos. is the average off-diagonal entry of the cosine matrix.
Effective rank is
\begin{align}
\mathrm{erank}
=
\exp\!\left(-\sum_i p_i\log p_i\right),
\qquad
p_i=\sigma_i\Big/\sum_j \sigma_j,
\end{align}
where \(\{\sigma_i\}\) are the singular values of the stacked expert-output
matrix. Higher effective rank indicates that expert outputs span a more diverse
subspace. Dead expert rate is the fraction of experts whose average routing mass
falls below a small threshold.

\begin{table}[h]
\caption{
Quantitative expert-specialization diagnostics on OfficeHome. Mean off-diagonal
cosine measures expert redundancy. Effective rank measures the diversity of the
expert-output subspace. Dead expert rate is the fraction of experts whose average
routing mass falls below a small threshold. Best results are shown in \textbf{bold}.
}
\label{tab:expert_specialization}
\centering
\small
\setlength{\tabcolsep}{5pt}
\begin{tabular}{lccc}
\toprule
Method
& Mean Offdiag Cos. $\downarrow$
& Effective Rank $\uparrow$
& Dead Expert Rate $\downarrow$ \\
\midrule
OMoE
& $0.35\ms{0.04}$
& $3.95\ms{0.12}$
& $0.45\ms{0.07}$ \\

GMoE
& $0.92\ms{0.03}$
& $3.64\ms{0.15}$
& $0.54\ms{0.06}$ \\

MESSI w/o $\mathcal L_{\mathtt{div}}$
& $0.50\ms{0.05}$
& $3.75\ms{0.14}$
& $0.40\ms{0.06}$ \\

\rowcolor{gray!10}
\textbf{MESSI (Ours)}
& $\mathbf{0.01}\ms{\mathbf{0.01}}$
& $\mathbf{4.21}\ms{\mathbf{0.11}}$
& $\mathbf{0.38}\ms{\mathbf{0.05}}$ \\
\bottomrule
\end{tabular}
\end{table}

\textbf{Results.}
Figure~\ref{fig:app_expert_analysis} shows a clear qualitative difference.
GMoE exhibits uniformly high off-diagonal cosine similarity, indicating that its
experts produce highly redundant representations. In contrast, MESSI yields a
near-diagonal matrix, showing that different experts encode substantially
different features.

Table~\ref{tab:expert_specialization} confirms this trend quantitatively. MESSI
reduces mean off-diagonal cosine similarity from \(0.92\) in GMoE to \(0.01\),
increases effective rank from \(3.64\) to \(4.21\), and lowers the dead expert
rate from \(0.54\) to \(0.38\). Removing \(\mathcal L_{\mathtt{div}}\) weakens
this effect: the off-diagonal cosine rises to \(0.50\) and the effective rank
drops to \(3.75\). OMoE improves over GMoE, but remains substantially less
specialized than the full MESSI model.

\textbf{Takeaway.}
These results show that MESSI does not improve simply by adding MoE capacity.
Instead, its routing and diversity objectives encourage experts to become both
distinct and active, yielding a less redundant and more expressive decomposition.

\subsection{Routing Diagnostics}
\label{app:routing_diagnostics}

We further analyze whether the learned router induces meaningful expert specialization rather than merely increasing model capacity. For each input, we record the routing distribution over experts from the last MoE block and aggregate it within each source domain. For a domain \(k\), we visualize three complementary routing statistics. First, \(P(e\vert d=k)\) measures the average expert usage of domain \(k\). Second, \(P_k(e\vert y)\) denotes the class-conditioned expert usage estimated only from samples in domain \(k\), i.e., \(P(e\vert y,d=k)\). Third, \(P(e\vert d=k,y)\) provides the same domain-conditioned class view in a format used to compare class-wise routing patterns across methods. These per-domain diagnostics reveal whether the router uses experts selectively, whether expert assignment is related to semantic classes, and whether the learned routing structure is dominated by domain-specific shortcuts.

Figures~\ref{fig:routing_diag_A}--\ref{fig:routing_diag_S} compare GMoE and MESSI on each PACS domain. Compared with GMoE, MESSI exhibits sharper and more class-structured routing patterns, indicating stronger expert selectivity and semantic specialization. At the same time, the class-conditioned routing visualizations allow us to check whether this specialization remains meaningful within each domain rather than collapsing into uniform expert usage or pure domain-level memorization.

\begin{figure*}[t]
  \centering

  \begin{subfigure}[t]{0.49\linewidth}
    \centering
    \includegraphics[width=\linewidth]{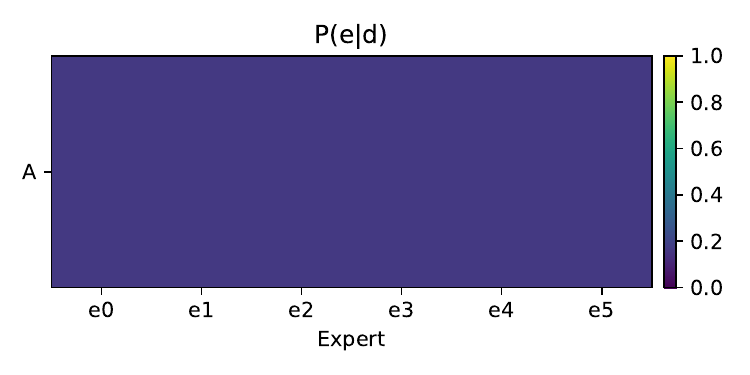}
    \caption{GMoE: \(P(e\vert d=A)\)}
  \end{subfigure}
  \hfill
  \begin{subfigure}[t]{0.49\linewidth}
    \centering
    \includegraphics[width=\linewidth]{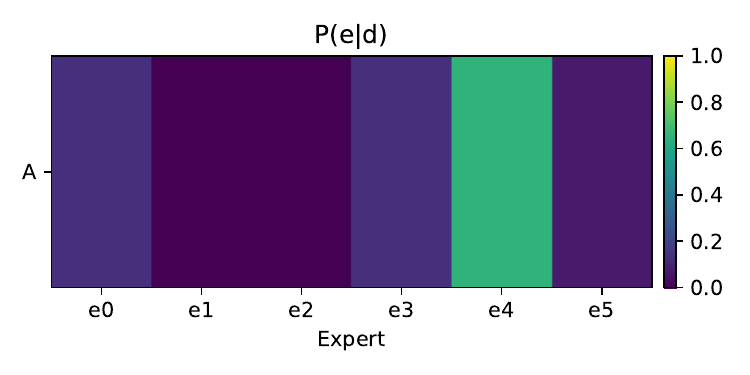}
    \caption{MESSI: \(P(e\vert d=A)\)}
  \end{subfigure}

  \medskip

  \begin{subfigure}[t]{0.49\linewidth}
    \centering
    \includegraphics[width=\linewidth]{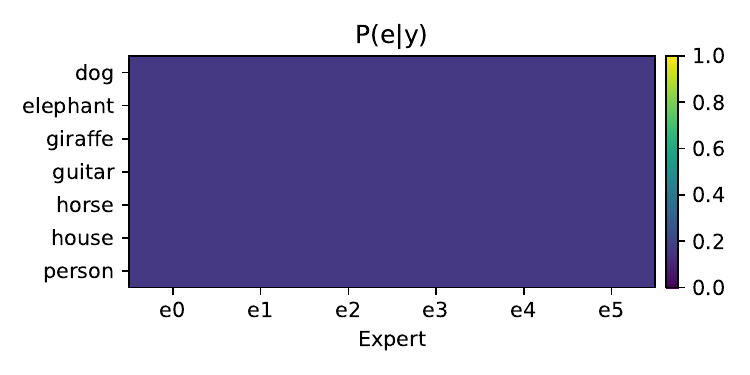}
    \caption{GMoE: \(P_A(e\vert y)\)}
  \end{subfigure}
  \hfill
  \begin{subfigure}[t]{0.49\linewidth}
    \centering
    \includegraphics[width=\linewidth]{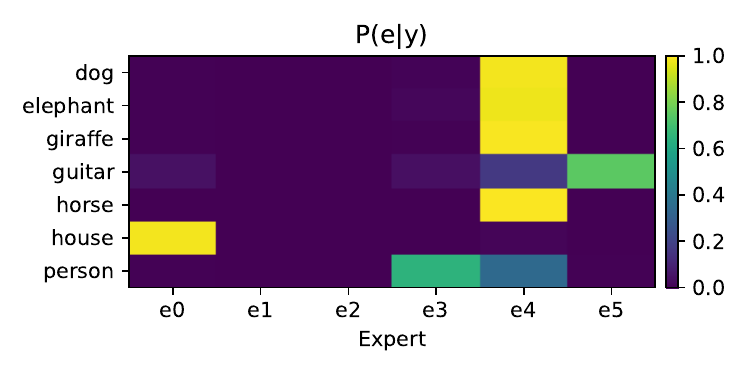}
    \caption{MESSI: \(P_A(e\vert y)\)}
  \end{subfigure}

  \medskip

  \begin{subfigure}[t]{0.49\linewidth}
    \centering
    \includegraphics[width=\linewidth]{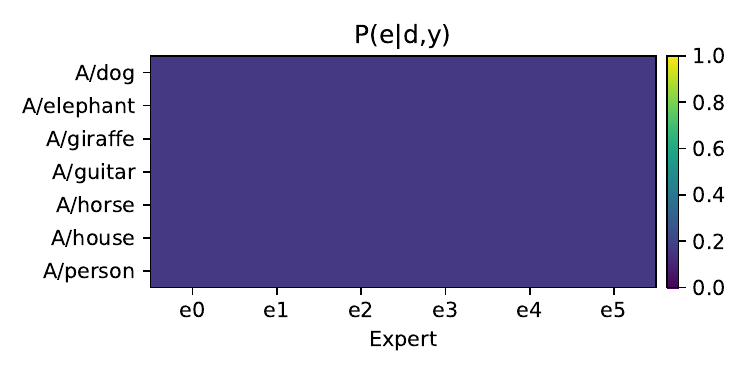}
    \caption{GMoE: \(P(e\vert d=A,y)\)}
  \end{subfigure}
  \hfill
  \begin{subfigure}[t]{0.49\linewidth}
    \centering
    \includegraphics[width=\linewidth]{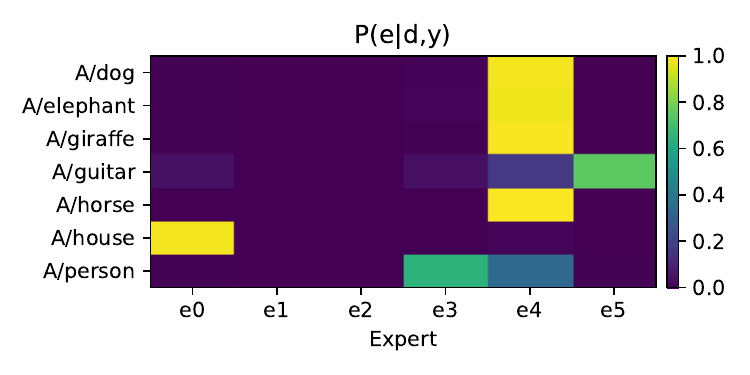}
    \caption{MESSI: \(P(e\vert d=A,y)\)}
  \end{subfigure}

  \caption{
  Routing diagnostics on the PACS Art domain. Rows correspond to routing aggregations and columns correspond to methods. \(P(e\vert d=A)\) measures domain-level expert usage. \(P_A(e\vert y)\) denotes class-conditioned expert usage computed only on domain \(A\). \(P(e\vert d=A,y)\) visualizes class-conditioned routing within the Art domain.
  }
  \label{fig:routing_diag_A}
\end{figure*}

\begin{figure*}[h]
  \centering

  \begin{subfigure}[t]{0.48\linewidth}
    \centering
    \includegraphics[width=\linewidth]{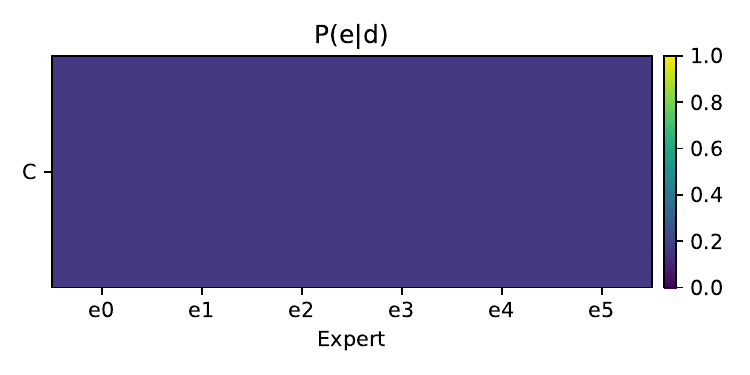}
    \caption{GMoE: \(P(e\vert d=C)\)}
  \end{subfigure}
  \hfill
  \begin{subfigure}[t]{0.48\linewidth}
    \centering
    \includegraphics[width=\linewidth]{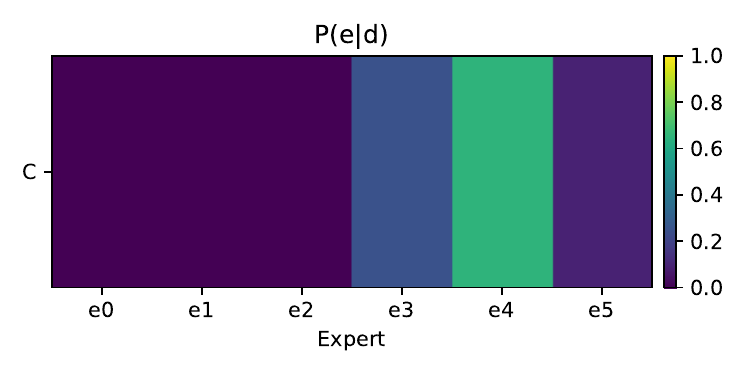}
    \caption{MESSI: \(P(e\vert d=C)\)}
  \end{subfigure}

  \medskip

  \begin{subfigure}[t]{0.48\linewidth}
    \centering
    \includegraphics[width=\linewidth]{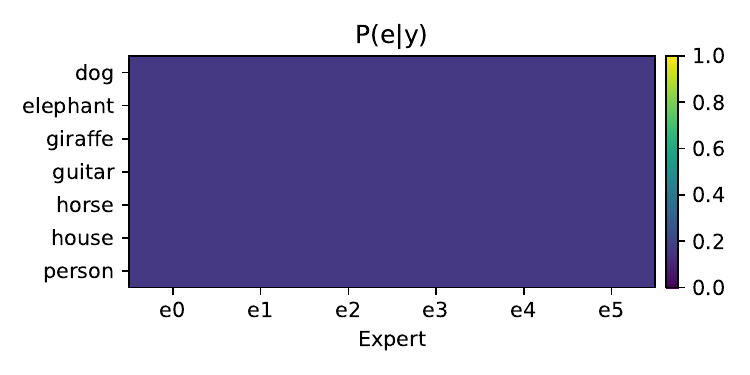}
    \caption{GMoE: \(P_C(e\vert y)\)}
  \end{subfigure}
  \hfill
  \begin{subfigure}[t]{0.48\linewidth}
    \centering
    \includegraphics[width=\linewidth]{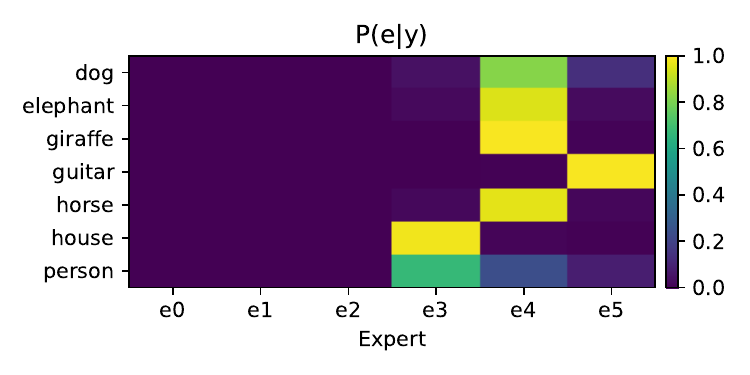}
    \caption{MESSI: \(P_C(e\vert y)\)}
  \end{subfigure}

  \medskip

  \begin{subfigure}[t]{0.48\linewidth}
    \centering
    \includegraphics[width=\linewidth]{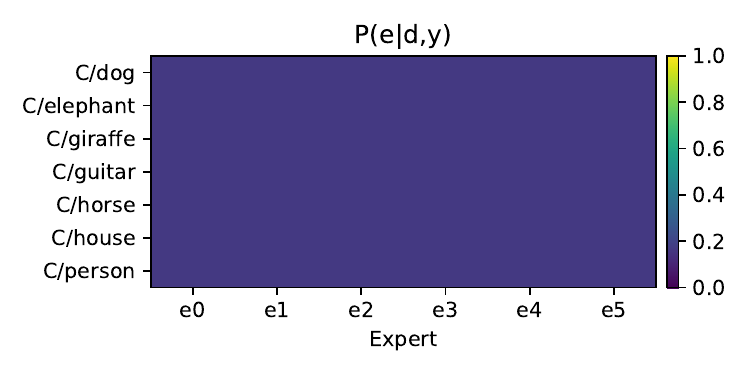}
    \caption{GMoE: \(P(e\vert d=C,y)\)}
  \end{subfigure}
  \hfill
  \begin{subfigure}[t]{0.48\linewidth}
    \centering
    \includegraphics[width=\linewidth]{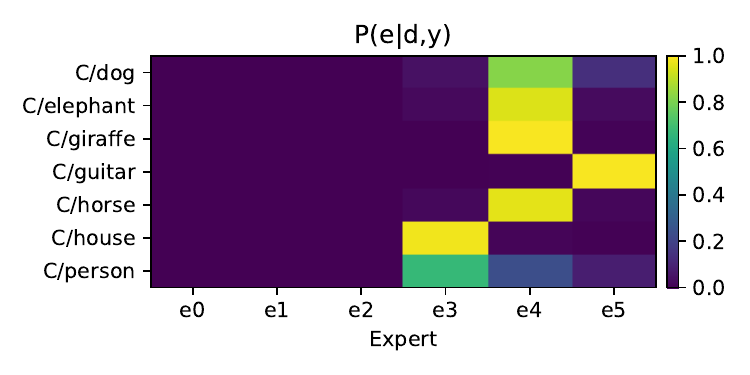}
    \caption{MESSI: \(P(e\vert d=C,y)\)}
  \end{subfigure}

  \caption{
  Routing diagnostics on the PACS Cartoon domain. Rows correspond to routing aggregations and columns correspond to methods. \(P(e\vert d=C)\) measures domain-level expert usage. \(P_C(e\vert y)\) denotes class-conditioned expert usage computed only on domain \(C\). \(P(e\vert d=C,y)\) visualizes class-conditioned routing within the Cartoon domain.
  }
  \label{fig:routing_diag_C}
\end{figure*}

\begin{figure*}[t]
  \centering

  \begin{subfigure}[t]{0.48\linewidth}
    \centering
    \includegraphics[width=\linewidth]{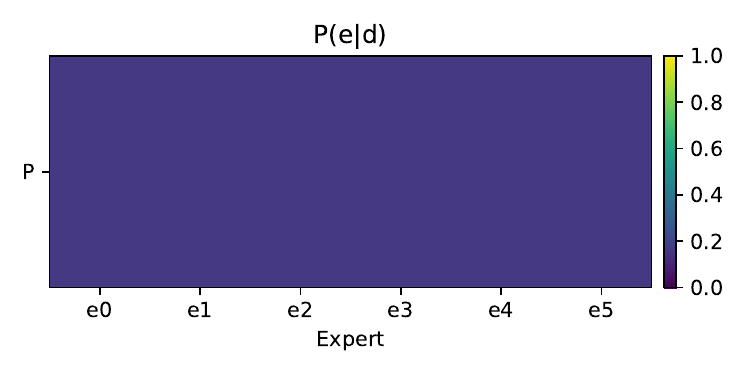}
  \end{subfigure}
  \hfill
  \begin{subfigure}[t]{0.48\linewidth}
    \centering
    \includegraphics[width=\linewidth]{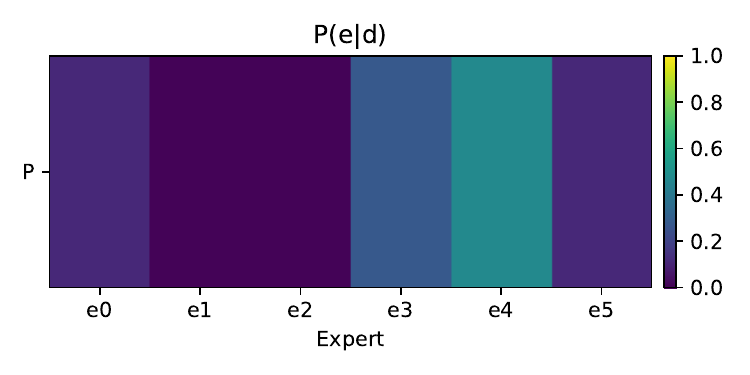}
  \end{subfigure}

  \medskip

  \begin{subfigure}[t]{0.48\linewidth}
    \centering
    \includegraphics[width=\linewidth]{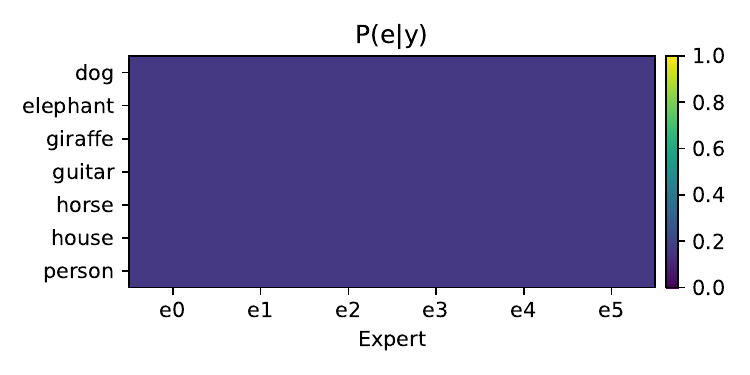}
  \end{subfigure}
  \hfill
  \begin{subfigure}[t]{0.48\linewidth}
    \centering
    \includegraphics[width=\linewidth]{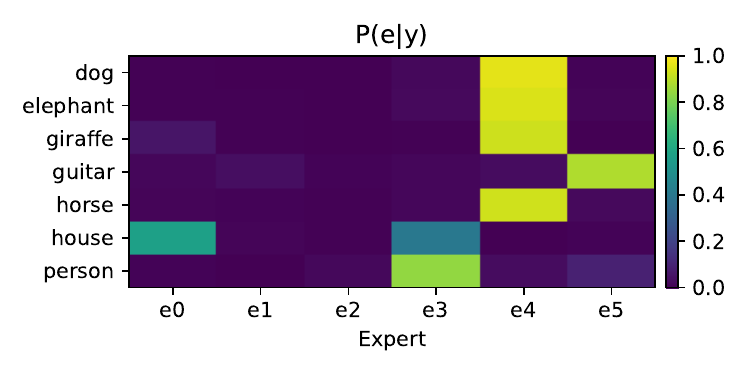}
  \end{subfigure}

  \medskip

  \begin{subfigure}[t]{0.48\linewidth}
    \centering
    \includegraphics[width=\linewidth]{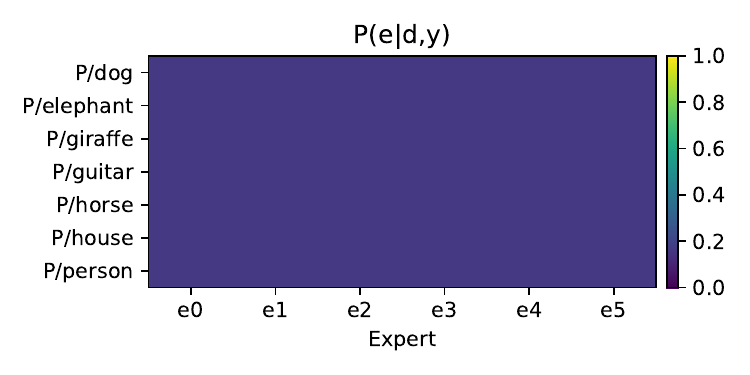}
  \end{subfigure}
  \hfill
  \begin{subfigure}[t]{0.48\linewidth}
    \centering
    \includegraphics[width=\linewidth]{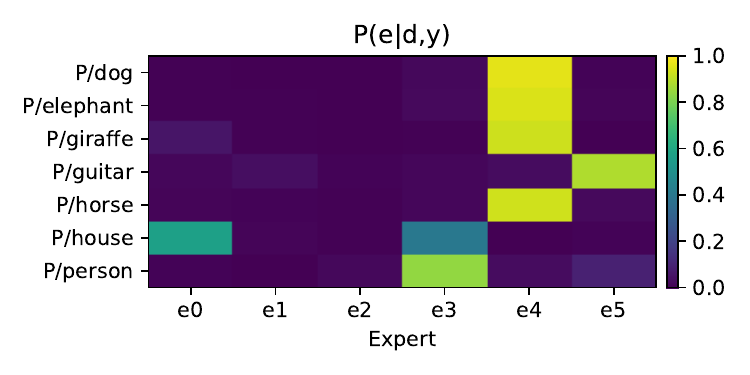}
  \end{subfigure}

  \caption{
  Routing diagnostics on the PACS Photo domain. In each row, the left panel shows GMoE and the right panel shows MESSI. From top to bottom, the rows show \(P(e\vert d=P)\), \(P_P(e\vert y)\), and \(P(e\vert d=P,y)\). These correspond to domain-level expert usage, class-conditioned expert usage computed on domain \(P\), and class-conditioned routing within the Photo domain, respectively.
  }
  \label{fig:routing_diag_P}
\end{figure*}

\begin{figure*}[t]
  \centering

  \begin{subfigure}[t]{0.48\linewidth}
    \centering
    \includegraphics[width=\linewidth]{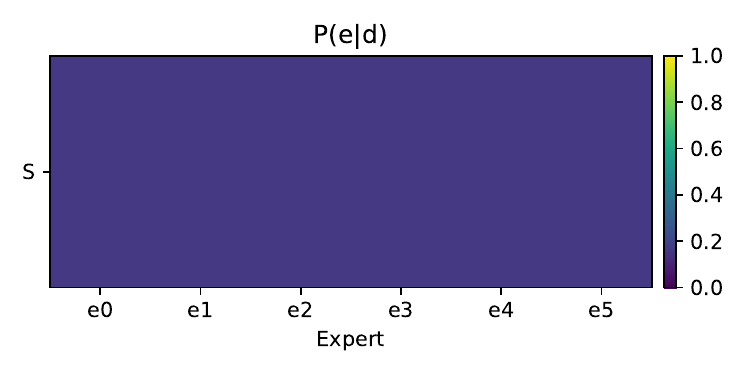}
  \end{subfigure}
  \hfill
  \begin{subfigure}[t]{0.48\linewidth}
    \centering
    \includegraphics[width=\linewidth]{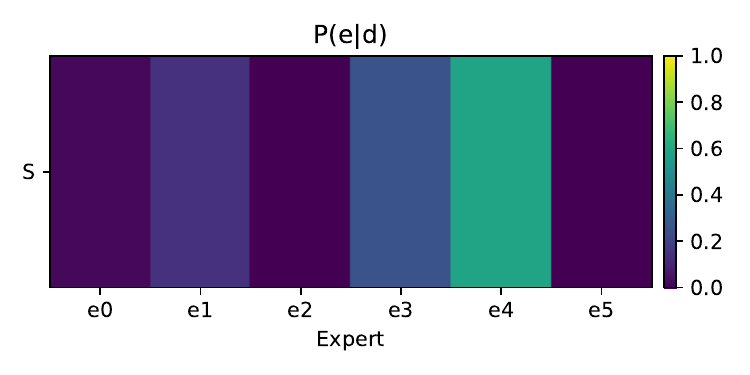}
  \end{subfigure}

  \medskip

  \begin{subfigure}[t]{0.48\linewidth}
    \centering
    \includegraphics[width=\linewidth]{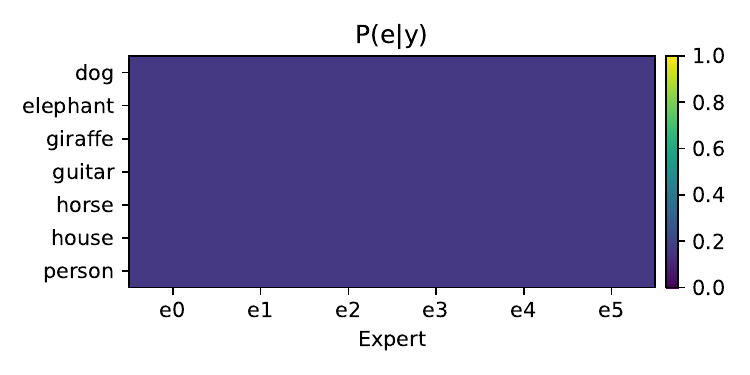}
  \end{subfigure}
  \hfill
  \begin{subfigure}[t]{0.48\linewidth}
    \centering
    \includegraphics[width=\linewidth]{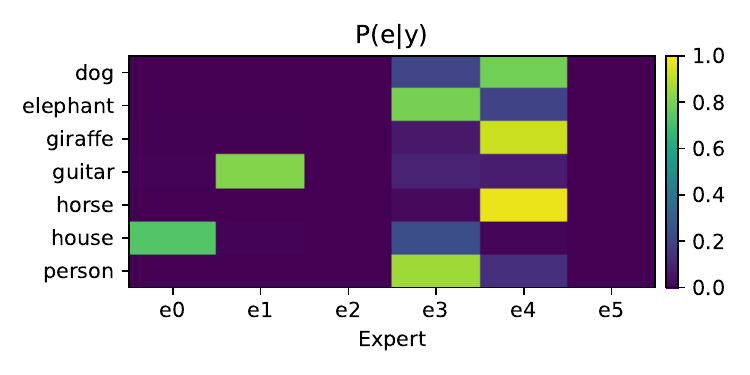}
  \end{subfigure}

  \medskip

  \begin{subfigure}[t]{0.48\linewidth}
    \centering
    \includegraphics[width=\linewidth]{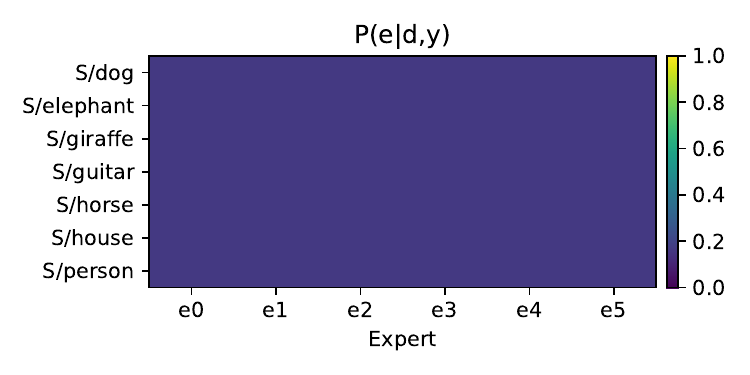}
  \end{subfigure}
  \hfill
  \begin{subfigure}[t]{0.48\linewidth}
    \centering
    \includegraphics[width=\linewidth]{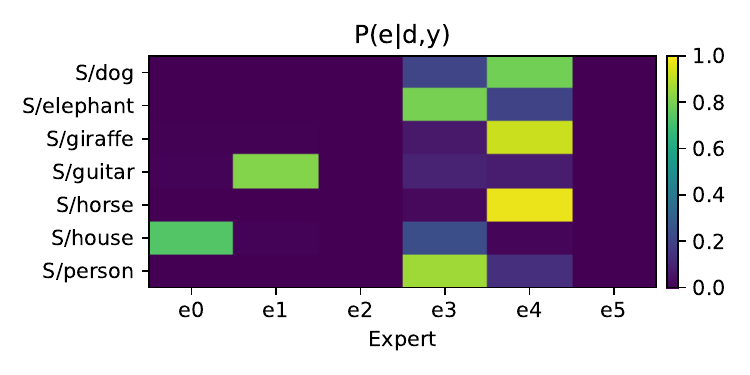}
  \end{subfigure}

  \caption{
  Routing diagnostics on the PACS Sketch domain. In each row, the left panel shows GMoE and the right panel shows MESSI. From top to bottom, the rows show \(P(e\vert d=S)\), \(P_S(e\vert y)\), and \(P(e\vert d=S,y)\). These correspond to domain-level expert usage, class-conditioned expert usage computed on domain \(S\), and class-conditioned routing within the Sketch domain, respectively.
  }
  \label{fig:routing_diag_S}
\end{figure*}

\clearpage
\subsection{Routing-Aware Alignment Controls}
\label{app:routing_aware_alignment}

The subset-conditioned invariance objective makes routing determine where alignment is applied. We test this mechanism by comparing MESSI with two controls. \textbf{Global-MoE} uses the same MoE architecture and auxiliary regularizers, but aligns all valid class-conditional source-domain pairs uniformly. \textbf{Random-Subset-MoE} uses the same sparsity level as MESSI, but selects aligned slots at random instead of using routing responsibility. All diagnostics are computed on held-out source-validation splits only; target-domain samples are not used.

For expert \(m\), source domain \(k\), and class \(c\), we compute the routing mass
\begin{align}
\rho_{k,c}^{(m)}
=
\frac{1}{|\mathcal D_{k,c}|}
\sum_{(x,y)\in\mathcal D_k,\,y=c}
\pi_m(x).
\end{align}
This induces a pair responsibility for each class-conditional domain pair:
\begin{align}
a_{ijc}^{(m)}
=
\sigma(\alpha\rho_{i,c}^{(m)})
\sigma(\alpha\rho_{j,c}^{(m)}),
\qquad i<j.
\end{align}
We call each tuple \((m,i,j,c)\) a \emph{slot}. A slot specifies which expert, domain pair, and class are considered for alignment.

\textbf{Responsibility similarity.}
For each method \(A\), we flatten the responsibilities \(\{a_{ijc}^{(m,A)}\}\) over all common valid slots into a vector \(\mathbf a_A\). We compare two methods using the Pearson correlation
\begin{align}
r(A,B)=\mathrm{corr}(\mathbf a_A,\mathbf a_B).
\end{align}
We also compare the top-\(q\%\) selected slots. Let \(S_A\) be the set of slots selected by method \(A\). The Jaccard overlap is
\begin{align}
J(A,B)
=
\frac{|S_A\cap S_B|}{|S_A\cup S_B|}.
\end{align}
High \(r(A,B)\) and \(J(A,B)\) indicate that two methods select similar class-domain-expert slots.

\textbf{Alignment on selected slots.}
A selected-versus-non-selected gap can be misleading because high-responsibility slots may already be easier to align. We therefore evaluate all methods on the same set of MESSI-selected slots. Let \(S_{\textsc{M}}\) be the top-\(q\%\) slots selected by MESSI. For each method \(A\), we compute
\begin{align}
\Delta_A(S_{\textsc{M}})
=
\frac{1}{|S_{\textsc{M}}|}
\sum_{(m,i,j,c)\in S_{\textsc{M}}}
D\!\left(
\mathcal Z_{i,c}^{(m,A)},\,
\mathcal Z_{j,c}^{(m,A)}
\right),
\end{align}
where \(D(\cdot,\cdot)\) is the empirical energy distance between the two class-conditional expert-feature distributions. Lower values indicate better alignment on the slots selected by MESSI.

\begin{table}[h]
\caption{
Routing-aware alignment diagnostics. \(r(A,B)\) is the Pearson correlation between pair-responsibility vectors, and \(J(A,B)\) is the Jaccard overlap between top-\(q\%\) selected-slot sets. \(\Delta_A(S_{\textsc{M}})\) is the feature discrepancy of method \(A\) evaluated on MESSI-selected slots; lower is better. Ctrl. \(\Delta\) averages Global-MoE and Random-Subset-MoE. Diagnostics use held-out source-validation splits only.
}
\label{tab:routing_aware_alignment}
\centering
\small
\setlength{\tabcolsep}{3.5pt}
\renewcommand{\arraystretch}{1.05}
\begin{tabular}{lcccccccc}
\toprule
Dataset
& \(r\)(G,R)
& \(J\)(G,R)
& \(r\)(M,G)
& \(J\)(M,G)
& \(r\)(M,R)
& \(J\)(M,R)
& Ctrl. \(\Delta\downarrow\)
& MESSI \(\Delta\downarrow\) \\
\midrule
PACS
& \(+1.00\)
& \(1.00\)
& \(-0.20\)
& \(0.04\)
& \(-0.20\)
& \(0.04\)
& \(0.235 \pm 0.014\)
& \(\mathbf{0.096 \pm 0.015}\) \\

OfficeHome
& \(+1.00\)
& \(0.93\)
& \(-0.20\)
& \(0.11\)
& \(-0.20\)
& \(0.11\)
& \(0.289 \pm 0.004\)
& \(\mathbf{0.217 \pm 0.006}\) \\

TerraIncognita
& \(+1.00\)
& \(0.95\)
& \(-0.20\)
& \(0.13\)
& \(-0.20\)
& \(0.14\)
& \(0.250 \pm 0.011\)
& \(\mathbf{0.179 \pm 0.026}\) \\

DomainNet
& \(+1.00\)
& \(0.95\)
& \(-0.20\)
& \(0.11\)
& \(-0.20\)
& \(0.10\)
& \(0.252 \pm 0.004\)
& \(\mathbf{0.143 \pm 0.032}\) \\
\bottomrule
\end{tabular}
\end{table}

\textbf{Results.}
Table~\ref{tab:routing_aware_alignment} shows two trends. First, Global-MoE and Random-Subset-MoE induce nearly identical responsibility patterns, with \(r(\mathrm{G},\mathrm{R})=1.00\) and high Jaccard overlap across datasets. In contrast, MESSI has weak or negative correlation with both controls and much lower selected-slot overlap. Thus, routing-aware alignment changes which expert-domain-class slots are selected for invariance, rather than simply sparsifying a global alignment objective.

Second, MESSI achieves lower discrepancy on the same MESSI-selected slots. The reduction is consistent across PACS, OfficeHome, TerraIncognita, and DomainNet. This indicates that the slots selected by MESSI are not only different from the controls, but are also better aligned after training. These diagnostics support the intended mechanism: routing determines where invariance is enforced, and the resulting selected slots receive stronger class-conditional alignment.

\subsection{Pairwise-to-Global Invariance Sweep}
\label{app:lambda_sweep}

We study how the invariance weight affects alignment, domain information, and target performance. On a diagnostic PACS split, we sweep
\(\lambda_{\mathrm{inv}}\in\{0,10^{-3},10^{-2},10^{-1},1,10\}\)
and report three quantities: pairwise class-conditional discrepancy, class-conditional domain-probe accuracy, and target-domain accuracy.

Figure~\ref{fig:lambda_sweep} and Table~\ref{tab:lambda_sweep} show three trends. First, increasing \(\lambda_{\mathrm{inv}}\) steadily reduces pairwise discrepancy, indicating stronger class-conditional alignment across source-domain pairs. Second, the class-conditional domain-probe accuracy also decreases, approaching the source-domain chance level at \(\lambda_{\mathrm{inv}}=10\). Thus, stronger invariance suppresses domain-predictive information in the learned representation. Third, target accuracy is not monotonic: it improves at moderate invariance strength, peaks at \(\lambda_{\mathrm{inv}}=10^{-2}\), and then degrades sharply. At \(\lambda_{\mathrm{inv}}=10\), target accuracy collapses even though pairwise discrepancy is nearly eliminated.

This sweep complements the theoretical result that enforcing pairwise conditional invariance over all domain pairs induces global conditional invariance. Empirically, however, stronger invariance is not always better. Moderate alignment improves generalization, whereas overly strong alignment removes label-relevant structure together with domain-specific variation. This supports the need for controlled, routing-conditioned alignment rather than uniformly enforcing global invariance.

\begin{figure}[h]
  \centering
  \includegraphics[width=\linewidth]{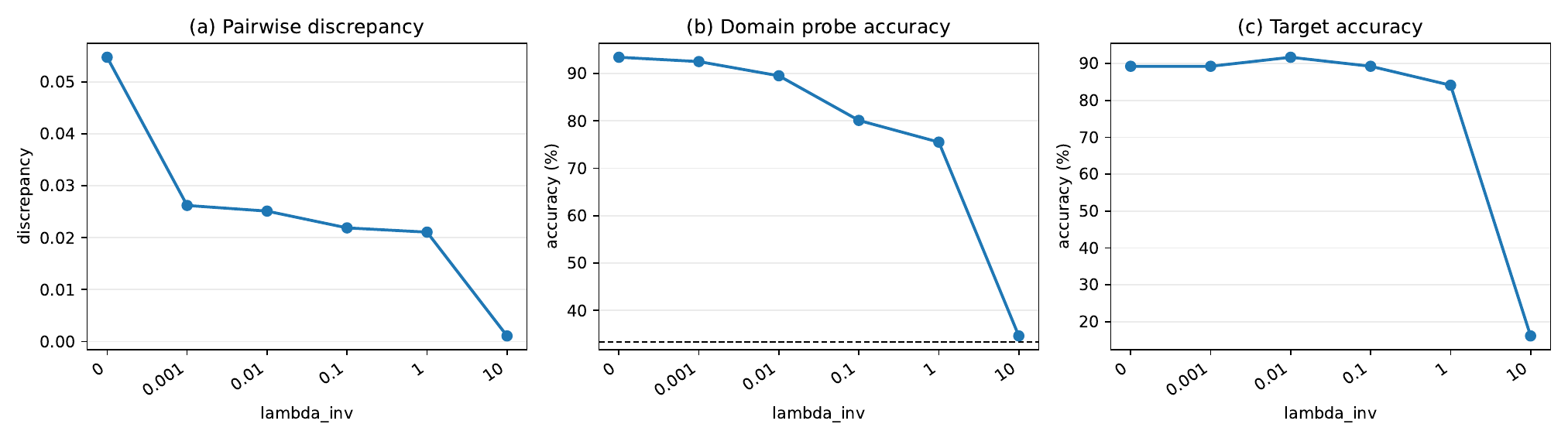}
  \caption{
    Pairwise-to-global invariance sweep on a PACS diagnostic split.
    We vary \(\lambda_{\mathrm{inv}}\) and report pairwise class-conditional discrepancy,
    class-conditional domain-probe accuracy, and target-domain accuracy.
    Increasing \(\lambda_{\mathrm{inv}}\) reduces pairwise discrepancy and domain predictability,
    but target accuracy peaks at moderate invariance strength and collapses under overly strong alignment.
    The dashed line in the domain-probe panel denotes the source-domain chance level.
    }
  \label{fig:lambda_sweep}
\end{figure}

\begin{table}[h]
\caption{
Numerical results for the pairwise-to-global invariance sweep on a PACS diagnostic split.
Moderate invariance improves target accuracy, with the best result at
\(\lambda_{\mathrm{inv}}=10^{-2}\).
In contrast, \(\lambda_{\mathrm{inv}}=10\) nearly eliminates pairwise discrepancy and
strongly suppresses domain predictability, but causes severe target-accuracy collapse.
}
\label{tab:lambda_sweep}
\centering
\small
\setlength{\tabcolsep}{6pt}
\begin{tabular}{rccc}
\toprule
\(\lambda_{\mathrm{inv}}\)
& Pairwise Disc. \(\downarrow\)
& Domain Probe Acc. \(\downarrow\)
& Target Acc. \(\uparrow\) \\
\midrule
0
& \(0.05475 \ms{0.0030}\)
& \(93.42 \ms{0.85}\)
& \(89.24 \ms{0.90}\) \\
\(10^{-3}\)
& \(0.02620 \ms{0.0022}\)
& \(92.51 \ms{0.95}\)
& \(89.24 \ms{0.85}\) \\
\(10^{-2}\)
& \(0.02511 \ms{0.0019}\)
& \(89.51 \ms{1.10}\)
& \(\textbf{91.69} \ms{\textbf{0.31}}\) \\
\(10^{-1}\)
& \(0.02188 \ms{0.0017}\)
& \(80.09 \ms{1.60}\)
& \(89.24 \ms{1.05}\) \\
1
& \(0.02106 \ms{0.0015}\)
& \(75.51 \ms{1.85}\)
& \(84.11 \ms{1.40}\) \\
10
& \(\textbf{0.00107} \ms{\textbf{0.0004}}\)
& $\mathbf{34.62} \ms{\mathbf{2.20}}$
& \(16.14 \ms{2.80}\) \\
\bottomrule
\end{tabular}
\end{table}

\section{Implementation Details}
\label{app:implementation}

\subsection{Architecture}
\label{app:architecture}

\textbf{Backbone.}
MESSI is evaluated with DeiT-Ti/16 and DeiT-S/16 encoders initialized from ImageNet-1K pretrained checkpoints~\citep{touvron2021training}. The DeiT-S/16 setting enables comparison with transformer-based and MoE-based baselines at the same encoder scale, whereas DeiT-Ti/16 tests whether subset-conditioned expert specialization remains effective with a lightweight encoder. In both settings, MESSI keeps the Transformer encoder unchanged and attaches a routing-based expert head to the final CLS representation. Thus, the backbone provides general visual features, while the expert head performs subset-conditioned specialization.

\textbf{Expert head.}
MESSI uses \(M=6\) experts. Let \(z=b_\phi(x)\) denote the final CLS representation produced by the encoder. Each expert maps \(z\) to a 384-dimensional representation with a two-layer MLP:
\begin{align}
h_m(z)=W_{m,2}\,\sigma(W_{m,1}z),
\end{align}
where \(\sigma\) is GELU. The router maps the same CLS representation to an expert distribution:
\begin{align}
\pi(x)=\mathrm{softmax}(W_{\mathrm r}z).
\end{align}
The final representation is the routed expert mixture
\begin{align}
h(x)=\sum_{m=1}^{M}\pi_m(x)h_m(z),
\end{align}
and the classifier is applied to \(h(x)\).

\textbf{Router input and domain usage.}
The router receives only the image representation \(z\). It does not receive the domain label. Domain labels are used only during training to compute source-domain alignment losses. At test time, MESSI requires neither domain labels nor target-domain data.

\subsection{Training Details}
\label{app:training}

\textbf{Backbone.}
All MESSI-S variants use a DeiT-Small/16, while MESSI-Ti uses DeiT-Ti/16 encoder initialized from an ImageNet-1K pretrained checkpoint~\citep{touvron2021training}. The encoder has model dimension \(384\). MESSI keeps the Transformer encoder unchanged and attaches a 6-expert mixture head to the final CLS representation, with expert output dimension \(384\).

\textbf{Optimizer and schedule.}
We train with Adam~\citep{kingma2014adam} and use a constant learning rate. The learning rate, weight decay, and number of training iterations are dataset-specific. These values are tuned once on source-validation splits following the DomainBed protocol and then fixed for all final runs:
\begin{center}
\small
\begin{tabular}{lccc}
\toprule
Dataset & Learning rate & Weight decay & Iterations \\
\midrule
PACS             & \(3\times10^{-5}\) & \(10^{-6}\) & 5{,}000  \\
OfficeHome       & \(1\times10^{-5}\) & \(10^{-6}\) & 10{,}000 \\
TerraIncognita   & \(5\times10^{-5}\) & \(10^{-4}\) & 10{,}000 \\
DomainNet        & \(5\times10^{-5}\) & \(0\)       & 15{,}000 \\
iWildCam (WILDS) & \(3\times10^{-5}\) & \(0\)       & 150{,}000 \\
\bottomrule
\end{tabular}
\end{center}

\textbf{Mini-batch construction.}
For PACS, OfficeHome, TerraIncognita, and DomainNet, we follow the leave-one-domain-out protocol. Each minibatch contains \(32\) images per source domain, giving total batch sizes of \(96\) for PACS, OfficeHome, and TerraIncognita, and \(160\) for DomainNet. For iWildCam, we treat camera-trap locations as domains. At each step, we sample \(K_{\mathrm{loc}}=4\) source locations and draw \(K_{\mathrm{img}}=8\) images from each location, forming a minibatch of \(32\) images. These active locations define the domain pairs used by \(\mathcal L_{\mathtt{ssi}}\).

\textbf{Augmentation.}
Training images use the following augmentation pipeline:
\begin{itemize}
    \item \texttt{RandomResizedCrop(224, scale=(0.7,1.0))},
    \item \texttt{RandomHorizontalFlip},
    \item \texttt{ColorJitter(0.3, 0.3, 0.3, 0.3)},
    \item \texttt{RandomGrayscale(p=0.1)}.
\end{itemize}
Images are then normalized with ImageNet statistics,
\(\mu=(0.485,0.456,0.406)\) and
\(\sigma=(0.229,0.224,0.225)\).
Validation and test images are resized to \(224\times224\) and normalized without augmentation.

\textbf{Model selection.}
For PACS, OfficeHome, TerraIncognita, and DomainNet, we use the DomainBed training-domain validation criterion. A stratified \(20\%\) split of each source domain is reserved for validation, and we select the checkpoint with the highest mean source-validation accuracy. Target-domain data are not used for training, hyperparameter tuning, or model selection. For iWildCam, we follow the WILDS protocol and select checkpoints using the official OOD validation split; final performance is reported on the OOD test split.

\subsection{Hyperparameters}
\label{app:hyperparameters}

All five datasets share the same MESSI objective weights and routing
configuration. We denote the sub-invariance loss as $\mathcal L_{\mathtt{ssi}}$
in the main method section; in the codebase the corresponding weight is
called \texttt{lambda\_inv}, i.e.\ $\lambda_{\mathtt{ssi}}\equiv\texttt{lambda\_inv}$.

\begin{center}
\small
\begin{tabular}{lc}
\toprule
Hyperparameter & Value \\
\midrule
$\lambda_{\mathtt{ssi}}$  & $0.01$ \\
$\lambda_{\mathrm{sp}}$   & $0.02$ \\
$\lambda_{\mathrm{bal}}$  & $0.02$ \\
$\lambda_{\mathrm{div}}$  & $0.02$ \\
Routing softmax temperature $\alpha$ & $4.0$ \\
Number of experts $M$ & $6$ \\
\bottomrule
\end{tabular}
\end{center}

\subsection{Notation and Hyperparameter Summary}
\label{app:notation_hparams}

\begin{table}[h]
\centering
\small
\caption{
Summary of notation used by MESSI. Domain indices refer to source domains in
one leave-one-domain-out training run. Target-domain samples are not used when
computing training losses.
}
\label{tab:notation}
\setlength{\tabcolsep}{4pt}
\renewcommand{\arraystretch}{1.08}
\begin{tabular}{ll}
\toprule
Symbol & Meaning \\
\midrule
\(K_s\) & Number of source domains in the current DG split \\
\(M\) & Number of experts \\
\(c\in\mathcal Y\) & Class index \\
\(i,j\in\{1,\ldots,K_s\}\) & Source-domain indices \\
\(b_\phi(x)\) & Shared encoder representation \\
\(h_m(b_\phi(x))\) & Representation produced by expert \(m\) \\
\(\pi_m(x)\) & Router probability assigned to expert \(m\) \\
\(\rho^{(m)}_{i,c}\) & Mean routing mass for expert \(m\), domain \(i\), class \(c\) \\
\(a^{(m)}_{ijc}\) & Routing-induced weight for aligning domains \(i,j\), class \(c\), expert \(m\) \\
\(\mathcal Z^{(m)}_{i,c}\) & Minibatch expert features from domain \(i\), class \(c\) \\
\(\mathcal C_B\) & Classes present in the current source minibatch \\
\(\mathcal P_B(c)\) & Source-domain pairs where class \(c\) appears in both domains \\
\(\mathcal L_{\mathtt{ssi}}\) & Subset-conditioned invariance loss \\
\(\mathcal L_{\mathtt{sp}}\) & Routing sparsity loss \\
\(\mathcal L_{\mathtt{bal}}\) & Expert load-balancing loss \\
\(\mathcal L_{\mathtt{div}}\) & Expert diversity loss \\
\bottomrule
\end{tabular}
\end{table}

\begin{table}[h]
\centering
\small
\caption{
Summary of implementation hyperparameters used by MESSI.
}
\label{tab:hparams}
\setlength{\tabcolsep}{4pt}
\renewcommand{\arraystretch}{1.08}
\begin{tabular}{ll}
\toprule
Hyperparameter & Role \\
\midrule
\(M\) & Number of experts \\
\(\alpha\) & Sharpness of routing-induced pair weights \(a^{(m)}_{ijc}\) \\
\(\epsilon_{\mathrm{OT}}\) & Entropic regularization coefficient in Sinkhorn OT \\
\(T_{\mathrm{OT}}\) & Number of Sinkhorn iterations \\
\(\lambda_{\mathtt{ssi}}\) & Weight of subset-conditioned invariance loss \\
\(\lambda_{\mathtt{sp}}\) & Weight of routing sparsity loss \\
\(\lambda_{\mathtt{bal}}\) & Weight of load-balancing loss \\
\(\lambda_{\mathtt{div}}\) & Weight of expert diversity loss \\
\(\beta_{\mathrm{bal}}\) & EMA coefficient for dataset-level routing load, if EMA balancing is used \\
\bottomrule
\end{tabular}
\end{table}

\subsection{Baseline Implementations}
\label{app:baseline_implementations}

We compare MESSI against three groups of baselines. The first group contains standard DomainBed baselines implemented with a ResNet-50 \citep{he2016deep} backbone. ERM serves as the empirical risk minimization reference. IRM~\citep{arjovsky2019invariant} represents invariant risk minimization. MMD~\citep{Li2018DomainGW} and CORAL~\citep{Sun2016DeepCC} represent distribution-alignment methods, where MMD matches kernel mean embeddings and CORAL matches second-order feature statistics. DANN~\citep{ganin2016domain} and CDANN~\citep{Li2018DeepDG} represent adversarial domain-alignment methods, with CDANN using class-conditional domain discrimination. MixStyle~\citep{Zhou2021DomainGW} represents feature-statistics-based data augmentation. Fish~\citep{Shi2021GradientMF} represents gradient-matching-based DG, and SWAD~\citep{Cha2021SWADDG} represents flat-minima-based model selection and weight averaging. Together, these methods form the conventional DomainBed reference set across ERM, invariant learning, feature alignment, adversarial alignment, augmentation, gradient matching, and flat-minima optimization. Unless otherwise stated, we use the DomainBed protocol and the reported or reproduced DomainBed-compatible implementations under the same train-domain validation model-selection rule.

The second group contains recent optimization and architecture baselines.
SAGM~\citep{Wang2023SharpnessAwareGM} represents sharpness-aware optimization
for DG. GMoE~\citep{Li2022SparseMA} represents sparse mixture-of-experts
architectures for domain generalization, and GMoE+SAGM combines the GMoE
backbone with the SAGM optimizer. LFME~\citep{Chen2024LFMEAS} is included as
a recent expert-based DG method based on multi-expert learning. These baselines
test whether the gains of MESSI can be explained by stronger optimization,
larger model capacity, or generic expert specialization.

The third group contains controlled MoE variants used for mechanism analysis.
These variants keep the MESSI backbone, number of experts, routing module,
training budget, and model-selection criterion fixed, while changing only the
training objective. This group includes classification-only MoE, globally
aligned MoE, random-subset-aligned MoE, and leave-one-out variants that remove
individual MESSI losses. These controls isolate the effect of routing-conditioned
subset alignment from MoE capacity, global alignment, random sparse alignment,
and auxiliary regularization.

\subsection{Model Size}
\label{app:model_size}

\begin{table}[h]
\centering
\footnotesize
\caption{
Trainable parameter counts in millions, excluding the final dataset-specific
classifier. Enc. denotes shared encoder parameters. MoE/Head denotes additional
expert, router, and method-specific head parameters. Values marked by
\(\approx\) are computed from the implemented architecture or from reported
model-size differences.
}
\label{tab:model_size}
\setlength{\tabcolsep}{4pt}
\renewcommand{\arraystretch}{1.08}
\begin{tabular*}{0.98\linewidth}{@{\extracolsep{\fill}}llccc@{}}
\toprule
Method & Backbone & Enc. & MoE/Head & Total \\
\midrule
ERM
& ResNet-50
& 25.60
& 0
& 25.60 \\

SAGM
& ResNet-50
& 25.60
& 0
& 25.60 \\

LFME
& ResNet-50
& 25.60
& 0
& 25.60 \\

GMoE
& DeiT-S/16
& 21.70
& \(\approx 12.10\)
& \(\approx 33.80\) \\

GMoE + SAGM
& DeiT-S/16
& 21.70
& \(\approx 12.10\)
& \(\approx 33.80\) \\

\rowcolor{gray!10}
MESSI-S
& DeiT-S/16
& 21.70
& \(\approx 7.09\)
& \(\approx 28.79\) \\

\rowcolor{gray!10}
MESSI-Ti
& DeiT-Ti/16
& 5.70
& \(\approx 1.78\)
& \(\approx 7.48\) \\
\bottomrule
\end{tabular*}
\end{table}

We exclude the final classifier because its size depends on the number of
classes and is small relative to the encoder and expert modules. SAGM adds no
trainable module, so it has the same parameter count as its base architecture;
for the same reason, GMoE+SAGM has the same size as GMoE. LFME is listed with
its ResNet-50 inference backbone for reference, although its training-time
teacher/expert construction differs from a single MoE inference model. The
GMoE MoE/Head count is estimated from the reported difference between GMoE-S/16
and ViT-S/16. MESSI head counts are computed from our implementation.

\textbf{Model-size discussion.}
Table~\ref{tab:model_size} shows that the gains of MESSI-S are not explained by
a larger parameter budget. Under the DeiT-S/16 encoder scale, MESSI-S has
approximately \(28.79\)M trainable parameters, compared with approximately
\(33.80\)M for GMoE-S/16. Thus, MESSI-S uses about \(15\%\) fewer trainable
parameters than GMoE-S/16 while still outperforming GMoE and OMoE in
Table~\ref{tab:main}. Parameter count alone does not establish the mechanism;
therefore, Appendix~\ref{app:objective_ablations} provides same-architecture
ablations where the MoE capacity is fixed and only the objective terms are
changed.

MESSI-Ti gives a smaller operating point. With a DeiT-Ti/16 encoder, the model
has approximately \(7.48\)M trainable parameters, making it about \(3.85\times\)
smaller than MESSI-S and \(4.52\times\) smaller than GMoE-S/16. This variant is
intended for settings where a lower parameter footprint is preferred, while
retaining the same routing-conditioned subset-alignment design.

\section{Limitations}\label{sec:limitations}
MESSI assumes that predictive structure is shared only across subsets of domains and that the router can recover useful routing patterns from data. When source domains are few, class coverage is limited or imbalanced, or routing fails to specialize, the resulting class-conditional routing-weighted discrepancies may become noisy and provide only weak alignment signals.

Conversely, when domains already admit a strong globally invariant representation, the additional routing decomposition may introduce unnecessary complexity and reduce statistical efficiency relative to simpler global alignment methods.

The method also introduces additional computational overhead due to pairwise class-conditional discrepancy estimation across experts and domains, particularly when using OT-based alignment objectives. Finally, the latent subsets induced by routing are not identifiable in general, and the learned routing assignments may not always correspond to semantically meaningful domain structure.

\newpage
\end{document}